\documentclass[lettersize,journal]{IEEEtran}
\usepackage{amsmath,amsfonts}
\usepackage{algorithmic}
\usepackage{algorithm}
\usepackage{array}
\usepackage[caption=false,font=normalsize,labelfont=sf,textfont=sf]{subfig}
\usepackage{textcomp}
\usepackage{stfloats}
\usepackage{url}
\usepackage{verbatim}
\usepackage{graphicx}
\usepackage{cite}
\usepackage{amsthm}
\usepackage{tabularray}

\usepackage{multirow}
\usepackage[table,xcdraw]{xcolor}
\usepackage[table,xcdraw]{xcolor}
\usepackage[normalem]{ulem}
\usepackage{graphicx}
\useunder{\uline}{\ul}{}

\usepackage[hidelinks,colorlinks=true,linkcolor=blue,citecolor=blue]{hyperref}

\newtheorem{theorem}{Theorem}[section] 
\newtheorem{definition}[theorem]{Definition}

\begin{document}

\title{QAEA-DR: A Unified Text Augmentation Framework for Dense Retrieval}


\author{
Hongming~Tan,
Shaoxiong~Zhan, 
Hai~Lin, 
Hai-Tao~Zheng,~\IEEEmembership{Senior~Member,~IEEE, }
and~Wai~Kin~(Victor)~Chan,~\IEEEmembership{Senior~Member,~IEEE}
\thanks{Corresponding author: Wai~Kin~(Victor)~Chan and Hai-Tao~Zheng.}
\thanks{Hongming~Tan, Hai~Lin, Hai-Tao~Zheng, and Wai~Kin~(Victor)~Chan are with Shenzhen International Graduate School, Tsinghua University, Shenzhen 518055, China, and also with Pengcheng Laboratory, Shenzhen, 518055, China. (e-mail: thm22@mails.tsinghua.edu.cn; ngyygm@outlook.com; zheng.haitao@sz.tsinghua.edu.cn; chanw@sz.tsinghua.edu.cn)}
\thanks{Shaoxiong~Zhan is with Shenzhen International Graduate School, Tsinghua University, Shenzhen 518055, China. (e-mail: jasaxion@gmail.com)}
}

\markboth{Journal of \LaTeX\ Class Files,~Vol.~14, No.~8, August~2021}%
{Shell \MakeLowercase{\textit{et al.}}: A Sample Article Using IEEEtran.cls for IEEE Journals}

\IEEEpubid{0000--0000/00\$00.00~\copyright~2021 IEEE}

\maketitle

\begin{abstract}
In dense retrieval, embedding long texts into dense vectors can result in information loss, leading to inaccurate query-text matching. Additionally, low-quality texts with excessive noise or sparse key information are unlikely to align well with relevant queries. Recent studies mainly focus on improving the sentence embedding model or retrieval process. In this work, we introduce a novel text augmentation framework for dense retrieval. This framework transforms raw documents into information-dense text formats, which supplement the original texts to effectively address the aforementioned issues without modifying embedding or retrieval methodologies. Two text representations are generated via large language models (LLMs) zero-shot prompting: question-answer pairs and element-driven events. We term this approach QAEA-DR: unifying question-answer generation and event extraction in a text augmentation framework for dense retrieval. To further enhance the quality of generated texts, a scoring-based evaluation and regeneration mechanism is introduced in LLM prompting. Our QAEA-DR model has a positive impact on dense retrieval, supported by both theoretical analysis and empirical experiments.
\end{abstract}


\section{Introduction}
\IEEEPARstart{D}ense retrieval \cite{lee2019latent, karpukhin2020dense} is a information retrieval method that uses text embeddings to find the relevant texts for a given query. In dense retrieval, sentence embeddings transform sentences into semantic vector representations, improving passage retrieval performance over word embeddings. 

A major challenge in dense retrieval is the risk of losing essential information when converting long texts into fixed-length dense vectors, as maintaining the fidelity of sparse representations for long texts often requires very high dimensions \cite{luan2021sparse}. Additionally, this limitation is emphasized in cases where the source texts are inundated with low-quality, noisy text, resulting in inconsistent retrieval quality. On one hand, recent works propose advanced retrievers or sentence embedding models to improve dense retrieval \cite{zhan2021jointly, li2023pseudo, ni2022sentence, yu2021improving, johnson2019billion, xiao2023c}. On the other hand, input enhancement for retrieval represents a distinct optimization strategy for retrieval tasks, including query transformation and data augmentation \cite{zhao2024retrieval}. Unlike query transformation \cite{wang2019query, gao2023precise}, data augmentation improves data quality before retrieval, enhancing performance without adding user wait time. However, in text retrieval, data augmentation methods typically focus on generating new query-text pairs for the retriever training \cite{bonifacio2022inpars}, rather than directly enhancing the original texts.
As a result, current data augmentation methods have not resolved inherent deficiencies in dense retrieval, specifically the loss of key information exacerbated by the presence of low-quality text. To address this issue, it is essential to consider data augmentation specifically applied to the retrieval text itself. Intuitively, we can enhance the original text by implementing \textit{text augmentation} methods to generate high-quality alternative texts, which concentrate key information to improve semantic similarity with the query.

\IEEEpubidadjcol
Taking inspiration from existing challenges and unexplored optimization strategies, we consider transforming raw texts into more \textit{information-dense} formats \cite{yang2014detecting} that present essential factual details concisely and directly for better dense retrieval. Specifically, we propose that dense retrieval can be improved through information extraction to generate new text embeddings, a text augmentation strategy that outperforms reliance on original text alone. These generated text embedding vectors achieve high fidelity by condensing information and removing noise, and they show higher similarity with the query vector than the original text vector. To implement this idea, we need to address three issues. (i) The first issue is \textit{What information extraction tools can effectively resolve the inherent challenges of key information loss and low-quality text in dense retrieval}? To address this issue, we focus on two high-level information extraction methods: question-answer generation (QAG) and event extraction (EE). 

Inspired by the longstanding tradition of Question Answering Systems (QAS) \cite{green1961baseball}, QA pairs should be the ideal text format for dense retrieval due to their high accuracy in providing precise responses to users' similar questions. QA pairs are information-dense as they focus on specific points from raw texts, presenting significant factual details directly and succinctly. This QA format aligns well with the query typically centered on a single topic, minimizing redundant information and offering a streamlined retrieval process for targeted inquiry \cite{cohen2023qa}. Additionally, studies indicate that QA pairs and documents can complement each other as knowledge bases \cite{lee2023read}, suggesting the incorporation of QA pairs into vector databases for enhanced dense retrieval. 

Additionally, we should consider event extraction as another crucial information extraction method based on knowledge graphs. Event extraction is a particularly information-dense form that extracts structured information from unstructured text to answer the ``5W1H'' questions (who, when, where, what, why, how) \cite{xiang2019survey}. It captures both entities and relationships, aiming to extract high-level information from text and present it in a structured, information-dense format. Consequently, events correspond to potential user ``5W1H'' queries and involve reorganizing and rewriting the original text to ensure precise information delivery, thus aligning semantically with these queries.

Furthermore, we observe that QA pairs and events possess both subtle connections and clear distinctions. Event-based knowledge representations share similarities with QA pairs: (1) They both capture \textit{high-level} semantic information at the sentence and paragraph levels rather than focusing solely on keywords and entities, providing deeper insights than keyword extraction or named entity recognition. (2) Each QA pair typically corresponds to an element in event representations, as events can answer ``5W1H'' questions. For example, ``when'' aligns with event time, ``where'' with location, and ``who'' or ``what'' with subjects or objects. 
Meanwhile, QA pairs and events differ fundamentally in structure. (1) QA pairs match individual information points and align with query semantics but each represents only a small portion of the source text, potentially limiting their ability to handle complex queries. (2) Events synthesize entities and relationships and incorporate various elements to potentially offer deeper and richer semantics than QA pairs. However, the lack of focus on a single information point in events reduces their alignment with queries. Therefore, we incorporate both QAG and EE into the text augmentation framework for their complementary benefits.

(ii) The second issue is \textit{What text generation model should be used}?
Our desired text generation model aims to: (1) effectively produce multiple QA pairs and events from any given raw text, with the quantity of generated outputs corresponding to the text's information content; (2) ideally manage all generation tasks within a unified model framework. In light of these requirements, we opt for large-scale pre-trained language models (LLMs, e.g., ChatGPT\footnote{https://chat.openai.com}) as text augmentation generators. Previous works in QAG and EE not only lack multilingual capabilities but also exhibit limited open-domain generalization, which distances them further from the ideal of a unified model framework. Unlike previous models, LLMs excel in text comprehension and generalization, enabling strong semantic understanding and information extraction capabilities. Despite the distinct nature of QAG and EE, LLMs could integrate these tasks into a unified framework that employs zero-shot prompting and supports multilingual data. We design prompt instructions for QAG and EE to generate JSON-formatted QA pairs and events.

Moreover, to ensure the output quality in unsupervised, training-free LLM generator, we introduce a penalty point system that deducts points based on specified criteria after the first generation. If scores fall below a predetermined threshold, we regenerate the text based on the deducted points to ensure enhanced text quality.

(iii) The last issue to address is \textit{How can the generated structured text be utilized for dense retrieval}? As a text augmentation method, our goal is to seamlessly add the generated structured texts into the datastore for retrieval. Initially, we should convert the structured text, previously output in JSON format by a large language model, back into unstructured natural language suitable for sentence embedding. We employ a straightforward conversion strategy: for QA pairs, we concatenate the question and answer to create one single text; for events, we sequentially combine all elements of the same event into one text. By converting back to unstructured text in this straightforward manner, we also explore different text organization strategies in our experiments. As a result, both the original and newly generated text chunks are embedded and incorporated into the final vector database. Importantly, we anticipate that the generated vectors will exhibit a higher similarity to the input query vectors than the original text vectors, thereby improving retrieval performance.

In this paper, based on the above discussion, we introduce \textbf{QAEA-DR}, a framework that integrates \textbf{Q}uestion-\textbf{A}nswer Generation (QAG) and \textbf{E}vent Extraction (EE) into a Text \textbf{A}ugmentation Framework for \textbf{D}ense \textbf{R}etrieval. QAEA-DR employs two types of generated text representations through LLM prompting: QA pairs and element-driven events. To further enhance the quality and robustness of text generation, we conduct scoring and text regeneration as the verification component in QAEA-DR. After generation, both QA pairs and events are converted back into unstructured texts. Subsequently, these generated texts are organized using two distinct text organization strategies and transformed into dense vectors. At last, these generated vectors are added to the vector database as high-quality retrieval vectors. Our experiments demonstrate that incorporating both event vectors and QA pair vectors into the vector database maximizes retrieval performance. In summary, the contributions of this paper are as follows:

\begin{itemize}
\item{To the best of our knowledge, QAEA-DR is the first comprehensive and universal text augmentation framework designed for dense retrieval.}
\item{QAEA-DR innovatively integrates the information extraction methods of QAG and EE into a unified framework of text generation and organization.}
\item{QAEA-DR employs an end-to-end LLM-based training-free text generator, integrating diverse prompts of generation and scoring-based output evaluation for high-quality and controllable text outputs.}
\item{QAEA-DR is evaluated through theoretical analysis and empirical validations on various embedding models and retrieval datasets to demonstrate its effectiveness and robustness.}
\end{itemize}

\section{Related Work}
In this section, we first review dense retrieval along with sentence embedding. Next, we discuss previous input enhancement methods for retrieval. Finally, we introduce some related works on information extraction.

\subsection{Dense Retrieval}
Dense retrieval has become an important research area following the development of pre-trained Transformer language models (PLMs) \cite{vaswani2017attention, devlin2019bert, karpukhin2020dense, raffel2020exploring, yates2021pretrained}. To enhance text retrieval performance, dense retrieval leverages PLM-based text embeddings to encode queries and documents into a shared semantic vector space, focusing on matching semantic contents beyond mere keywords. This text embedding application in retrieval is fundamental to Retrieval-Augmented Generation (RAG) \cite{lewis2020retrieval}, which reduces the hallucinations in LLMs. Recent advancements in dense retrieval include architectural innovations, optimized training methodologies, and efficient indexing techniques, all of which contribute to improved retrieval accuracy and efficiency \cite{zhan2021jointly, li2023pseudo, ni2022sentence, yu2021improving, johnson2019billion, zheng2020fast}. Since the introduction of Sentence-BERT \cite{reimers2019sentence} and Dense Passage Retrieval (DPR) \cite{karpukhin2020dense}, numerous sentence embedding models have been proposed to enhance dense passage retrieval. Advanced sentence embedding models, which have been highlighted in the retrieval task of massive text embedding benchmark (MTEB) \cite{muennighoff2023mteb}, include Contriever \cite{izacard2021unsupervised}, M3E\footnote{https://huggingface.co/moka-ai/m3e-base}, BGE \cite{xiao2023c}, etc. Our text augmentation method serves as a preprocessing module for dense retrieval and is compatible with various embedding models mentioned above.

\subsection{Input Enhancement in Retrieval}
In addition to the optimization methods for the retriever, input enhancement strategy represents a distinct optimization approach for retrieval tasks \cite{zhao2024retrieval}. In particular, input data to a retriever includes user query and datastore. Therefore, input enhancement for retrieval can be categorized into two types: query transformation and data augmentation. Query transformation modifies the input query during retrieval, for example, Hypothetical Document Embeddings (HyDE) \cite{gao2023precise} that generate pseudo documents from queries, and KNN-based Query Expansion (KNN-QE) \cite{wang2019query} that enhances queries using local conceptual word embeddings. Data augmentation improves the data to be retrieved before the retrieval process, including synthesizing data, clarifying ambiguities, updating outdated data, etc. Compared to query transformation, data augmentation models have the advantage of not consuming user waiting time in retrieval, which is particularly important in practical applications. Mainstream studies focus on data augmentation for cross-modal retrieval, such as Make-An-Audio \cite{huang2023make} and ReACC \cite{lu2022reacc}. In terms of text-to-text retrieval, methods like InPars \cite{bonifacio2022inpars} generates new query-text pairs as training data. As current data augmentation methods do not consider enhancing the original text in text retrieval, we propose a text augmentation framework in this paper.

\subsection{Information Extraction}
Information extraction (IE) automatically isolates text fragments and extracts structured data from unstructured sources through NLP \cite{cowie1996information}. On one hand, IE is integral to constructing knowledge graphs (KGs), which have attracted considerable attention as a structured form of knowledge \cite{ji2021survey}. Tasks related to KG-based IE include named entity recognition, relation extraction, and event extraction \cite{wadden2019entity}. In particular, event extraction (EE) captures both entities and relationships to extract high-level structured information from raw text.
\begin{itemize}
\item{EE has evolved from rule-based approaches \cite{valenzuela2015domain} to deep learning methods like Dynamic Multi-Pooling Convolutional Neural Networks (DMCNN) \cite{chen2015event} and Joint Event Extraction via Recurrent Neural Networks (JRNN) \cite{nguyen2016joint}, and more recently to ChatGPT for Event Extraction (ChatEE) \cite{wei2023zero}, and the Generative Template-based Event Extraction (GREE) \cite{huang2023event}, reflecting significant progress in the field.}
\item{There are methods that achieve multi-event extraction, such as Jointly Multiple Event Extraction (JMEE) \cite{liu2018jointly} and Proxy Nodes Clustering Network (ProCNet) \cite{wang2023document}. Nevertheless, current multi-event extraction methods are closed-domain and are limited by their reliance on predefined event schemas.}
\end{itemize}
On the other hand, Question-Answer Generation (QAG), an extension of Question Generation (QG) \cite{lewis2019unsupervised, zhang2019addressing}, generates several QA pairs given a text. Notably, QAG can also be classified as IE since QA pairs are structured texts. 
\begin{itemize}
\item{QAG have progressed from rule-based models \cite{heilman2010good} to generative-based PLMs like Information-Maximizing Hierarchical Conditional VAEs (Info-HCVAE) \cite{lee2020generating} and Language Models for Question and Answer Generation (LMQG) \cite{ushio2023empirical}.}
\end{itemize}

However, current multi-event extraction and QAG face issues such as a lack of multilingual support, uncontrollable generation quantity and quality, and incompatibility within a unified model framework. Although these models show relatively good results on some datasets, they are far from our goal of a generalizable, quality-controllable, and unsupervised unified framework in open-domain applications. QAEA-DR combines EE and QAG, leveraging LLM to build an end-to-end framework that involves prompt-based generation, evaluation, and regeneration.

\begin{figure*}[t]
  \centering
  \includegraphics[width=0.9\textwidth]{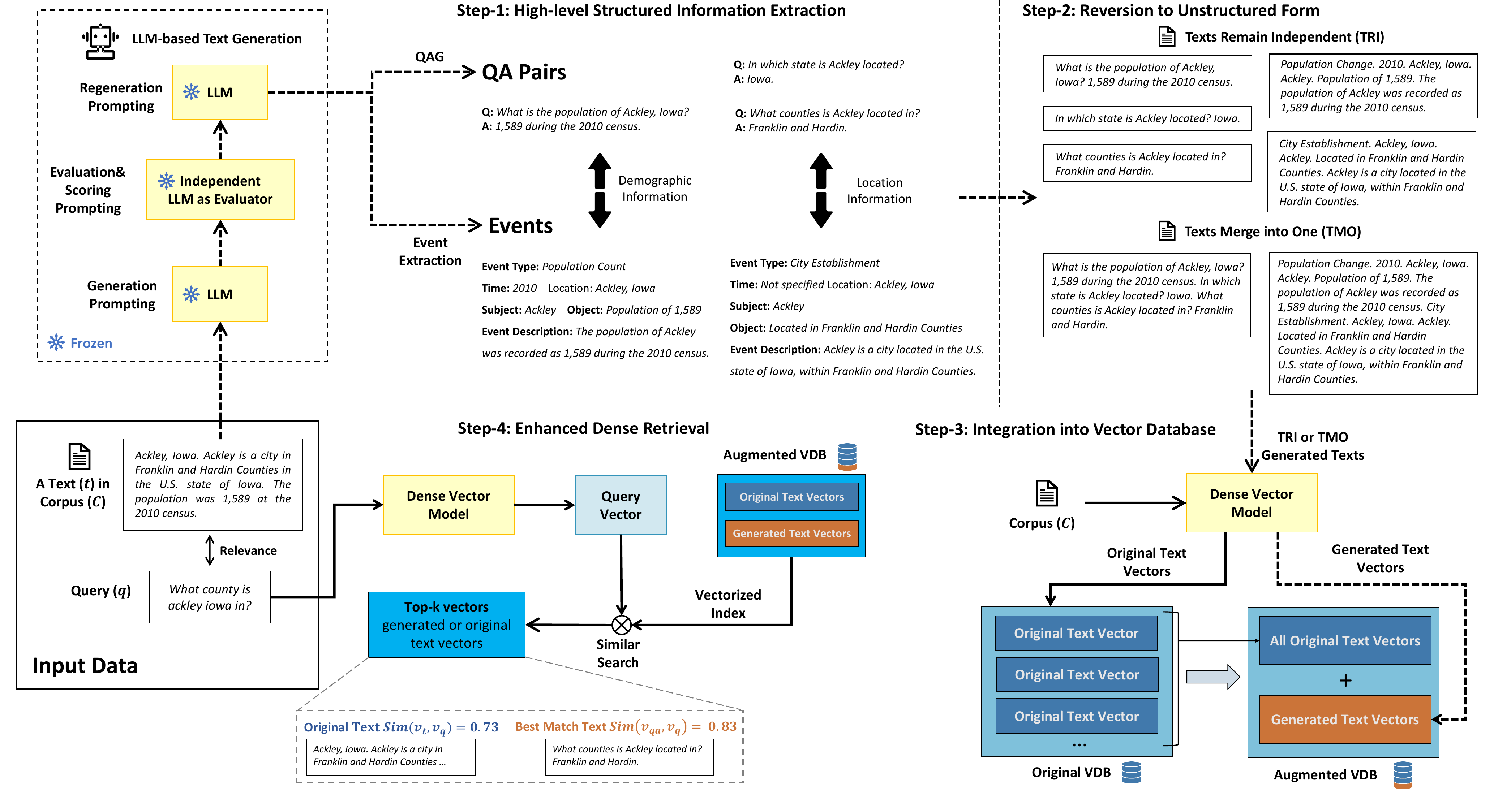}
  \caption{QAEA-DR example. The dashed arrows represent the text augmentation path of QAEA. In Step-1: High-level Structured Information Extraction through LLM-based text generators with frozen parameters, the generated QA pairs and events preserve similar key information in different formats. This is followed by Step-2: Reversion to Unstructured Form and Step-3: Integration into the Vector Database. In Step-4: Enhanced Dense Retrieval, the results demonstrate that there exists a generated vector with higher query relevance compared to the original text vector. This is because the key information density of the generated text vector is enhanced through information extraction, making it semantically closer to the query.}
  \label{QAEA-DR Framework}
\end{figure*}

\section{Approach}

\subsection{Notations and Problem Definition}
In this paper, we focus on text augmentation approach for dense passage retrieval. A retrieval dataset typically comprises three types of data: the corpus, queries, and labeled query-text relationships. Initially, let \(\mathcal{C} = \{t_1, t_2, \ldots, t_n\}\) represent a corpus, where each \(t_i\) is a text chunk (simplified as text in the following discussion) and \(n\) is the total number of texts in corpus. The initial step in dense passage retrieval is to construct a mapping function \(\Phi: \mathcal{C} \rightarrow \mathbb{R}^d\), where \(d\) is the vector dimension, such that semantically similar texts are close in the vector space. Specifically, the function \(\Phi\) uses sentence embedding model to transform all texts into dense vectors (i.e. embedding) stored in a vector database. We denote the resulting vector database as \(\text{VDB}_\text{ori}\), where \(\text{VDB}_\text{ori} = \{v_1, v_2, \ldots, v_n\}\) with each vector \(v_i \in \text{VDB}_\text{ori}\) corresponding to a text \(t_i\) in \(\mathcal{C}\). Given a query text \(q\), which is also mapped to a vector \(v_q \in \mathbb{R}^d\) by \(\Phi\), the retriever calculates the top-k vectors \(v_i \in \text{VDB}_\text{ori}\) with the highest similarity to query vector \(v_q\), resulting in a subset \(S \subseteq \text{VDB}_\text{ori}\), where \(|S| = k\). The vector similarity, denoted as \(\text{sim}(v_q, v_i)\), measures the distance between \(v_q\) and each vector \(v_i\) in \(\text{VDB}_\text{ori}\) (e.g., calculating the cosine similarity based on the inner product \(\langle v_q, v_i \rangle\)). Evaluation metrics (e.g., NDCG) are used to calculate retrieval scores based on labeled query-text relationships. Here, we define our QAEA-DR as follows:

\begin{definition}[QAEA-DR]
QAEA-DR is a text augmentation framework that augments the original corpus \(\mathcal{C}\) by generating QA pairs and element-driven events using LLM-based generators. This process enriches the vector database \(\text{VDB}_\text{ori}\) by adding new vector representations derived from the augmented texts. The similarity of the query to generated text vectors should exceed that of original text vectors, potentially enhancing retrieval quality.
\end{definition}

\subsection{Overview of QAEA-DR}
\label{sec:overview}
Fig. \ref{QAEA-DR Framework} shows the complete workflow of QAEA-DR, illustrating an example of the framework in action. Specifically, QAEA-DR operates as follows:
\begin{itemize}
\item \textit{Step-1: Structured Information Extraction.} Each text \(t_i\) from the corpus \(\mathcal{C}\), where \(i=1, \ldots, n\), is augmented using LLM prompting to generate JSON format QA pairs \(\text{QA}_{\text{json}}\) and events \(\text{EVENT}_{\text{json}}\). We discuss the design of LLM prompts for structured text augmentation in Section \ref{Prompt-based Generation}. As illustrated in Fig. \ref{QAEA-DR Framework}, both types of structured texts effectively extract key information. Specifically, each QA pair presents an individual information point, while each event summarizes multiple points.
\item \textit{Step-2: Reversion to Unstructured Form.} The generated structured texts are converted back into unstructured natural language texts, resulting in a set of QA texts \(\{qa_i^{(1)}, qa_i^{(2)}, \ldots, qa_i^{(l)}\}\) and a set of event texts \(\{event_i^{(1)}, event_i^{(2)}, \ldots, event_i^{(m)}\}\), where \(l\) and \(m\) represent the total number of QA pair texts and event texts generated from \(t_i\), respectively. Subsequently, we mainly consider two text organization strategies:
  \begin{itemize}
      \item \textit{Texts Remain Independent (TRI):} Maintain the generated set as \( \text{QA}_i = \{qa_i^{(1)}, qa_i^{(2)}, \ldots, qa_i^{(l)}\} \) and \( \text{EVENT}_i = \{event_i^{(1)}, event_i^{(2)}, \ldots, event_i^{(m)}\} \), for \(i=1, \ldots, n\).
      \item \textit{Texts Merge into One (TMO):} In this mode, individual texts generated from the same original text \(t_i\) are concatenated, forming singleton set \( \text{QA}_i = \{qa_i: qa_i^{(1)} + qa_i^{(2)} + \ldots + qa_i^{(l)}\} \) and \( \text{EVENT}_i = \{event_i: event_i^{(1)} + event_i^{(2)} + \ldots + event_i^{(m)}\} \), for \(i=1, \ldots, n\). The ``+'' denotes text concatenation.
  \end{itemize}
Overall, TRI decomposes texts, retaining all segments extracted from the original, but may include noisy texts unrelated to all queries. Conversely, TMO consolidates generated texts to reduce the density of noise.
\item \textit{Step-3: Integration into Vector Database.} 
The transformed texts are mapped to vectors by the function \(\Phi\). 
For QA texts, \( V_{\text{QA}_i} = \Phi(\text{QA}_i) = \{v_{qa_i}^{(j)} \mid j = 1, \ldots, l\} \) (TRI) or \( \{v_{qa_i}\} \) (TMO), where \(i=1, \ldots, n\), results in the vector database \(\text{VDB}_{\text{QA}} = \bigcup_{i=1}^{n} V_{\text{QA}_i}\). Similarly, \(V_{\text{EVENT}_i} = \Phi(\text{EVENT}_i) = \{v_{event_i}^{(j)} \mid k = 1, \ldots, m\} \) (TRI) or \( \{v_{event_i}\} \) (TMO), populates \(\text{VDB}_{\text{EVENT}} = \bigcup_{i=1}^{n} V_{\text{EVENT}_i}\). These generated vectors are then integrated into the final vector database \(\text{VDB}_\text{final} = \text{VDB}_\text{ori} + \text{VDB}_\text{QA} + \text{VDB}_\text{EVENT}\) to augment the original vector space.
\item \textit{Step-4: Enhanced Dense Retrieval.} The query vector \(v_q\) searches for the top-k similarity vectors in \(\text{VDB}_\text{final}\). For any given query \(q\) associated with a positively related text \(t_i\), there exists a vector in either \(\text{VDB}_\text{QA}\) or \(\text{VDB}_\text{EVENT}\) that exhibits higher similarity with the query vector \(v_q\) than the original text vector \(v_i\).
\end{itemize}

In conclusion, our main contribution is the implementation of QAEA-DR, which, in Step-1, Step-2, and Step-3, generates two new types of text vectors—QA pair vectors and event vectors—and integrates them into the vector database. These generated vectors enhance the retrieval performance in the final Step-4 of dense retrieval. Fig. \ref{QAEA-DR Framework} demonstrates that in Step-4, the best match with the query is derived from the generated vectors with high similarity. 

In the following sections, we will first discuss \textit{text augmentation details in Step 1 to 3}. Then, we substantiate the effectiveness of QAEA-DR theoretically in the subsequent section, addressing \textit{why generated vectors in Step-4 could exhibit higher similarity with the query vector than the original text vector}.

\subsection{LLM-based Text Augmentation in QAEA}
\label{Prompt-based Generation}
In this section, we describe our implementation of LLM-based text augmentation and the unifying properties of QAEA. \textit{QAEA} is defined as a text augmentation module excluding the retrieval component. It combines original texts, QA pairs, and events into a new vector database to enhance natural texts through information extraction. QAEA corresponds to Steps 1 to 3 in Fig. \ref{QAEA-DR Framework}.

Following standard LLM-based prompt engineering practices \cite{giray2023prompt}, our defined single-step zero-shot prompt consists of three components: instruction, input data, and output indicator. For both QAG and EE prompting tasks, we make targeted adjustments to the prompt instructions. Fig. \ref{prompt_template} illustrates the three-step prompts defined for both QAG and EE, including generation, scoring-based quality evaluation, and regeneration. We achieve different functionalities by modifying the instructions for each type of prompt.

\textit{Question-Answer Generation.} In QAG, our goal is to generate as many informative structured QA pairs as possible through instruction. Due to the lack of a universally recognized question generation directive, we employ question categorization to guide the LLM in producing diverse QA pairs. In designing question types, we observe that rhetorical patterns in writing (e.g., cause and effect) serve as methods for organizing information and can be generalized for question categorization. Consequently, the content directive specifies five question types for varied outputs: \textit{factual inquiry}, \textit{explanation and definition}, \textit{cause and effect}, \textit{comparison and contrast}, and \textit{evaluation and opinion}. Additionally, the prompt instructs LLM to highlight frequently occurring entities and relationships in the original text, which should be reflected in the generated QA pairs. Regarding the output format, the instructions guide the LLM to produce QA pairs in JSON format \(\text{QA}_{\text{json}}\): \{\textit{``question type''}: [[\textit{``question''}, \textit{``answer''}]]\}, where ``question type'' includes five categories, each capable of containing multiple QA pairs depending on LLM generation. As illustrated in Step-2 of Fig. \ref{QAEA-DR Framework}, the output of QAG is processed by simply concatenating the ``question'' and ``answer'' strings into natural language texts.

\textit{Event Extraction.}
Since our text augmentation method for dense retrieval is initially designed for application on a small-scale Chinese news passage retrieval dataset (sCNPR) we created from a scientific project, we naturally consider event extraction. Unlike previous zero-shot prompt-based ChatEE \cite{wei2023zero}, which requires predefined event types and supports single-event extraction, our approach allows the LLM to detect and generate multiple event types from the original text. In the EE prompt instructions, we first direct the LLM to identify multiple event types and use these generated types as triggers to populate event elements. Drawing from the event element categorization in the ACE 2005 dataset and common real-world event attributes, we define that each event includes the elements ``event type,'' ``time,'' ``location,'' ``event subject,'' ``event object,'' ``event,'' and ``impact.'' In terms of the output format, we guide the LLM to generate event outputs in JSON format \(\text{EVENT}_{\text{json}}\): [\textit{\{``event type'', ``time'', ``location'', ``event subject'', ``event object'', ``event'', ``impact''\}}], where outputs default to null if corresponding event elements are absent in the original text. Fig. \ref{QAEA-DR Framework} shows that the output of EE is transformed back into unstructured text by concatenating all event elements within an event, separated by periods.

\begin{figure} 
  \centering
  \includegraphics[width=0.8\linewidth]{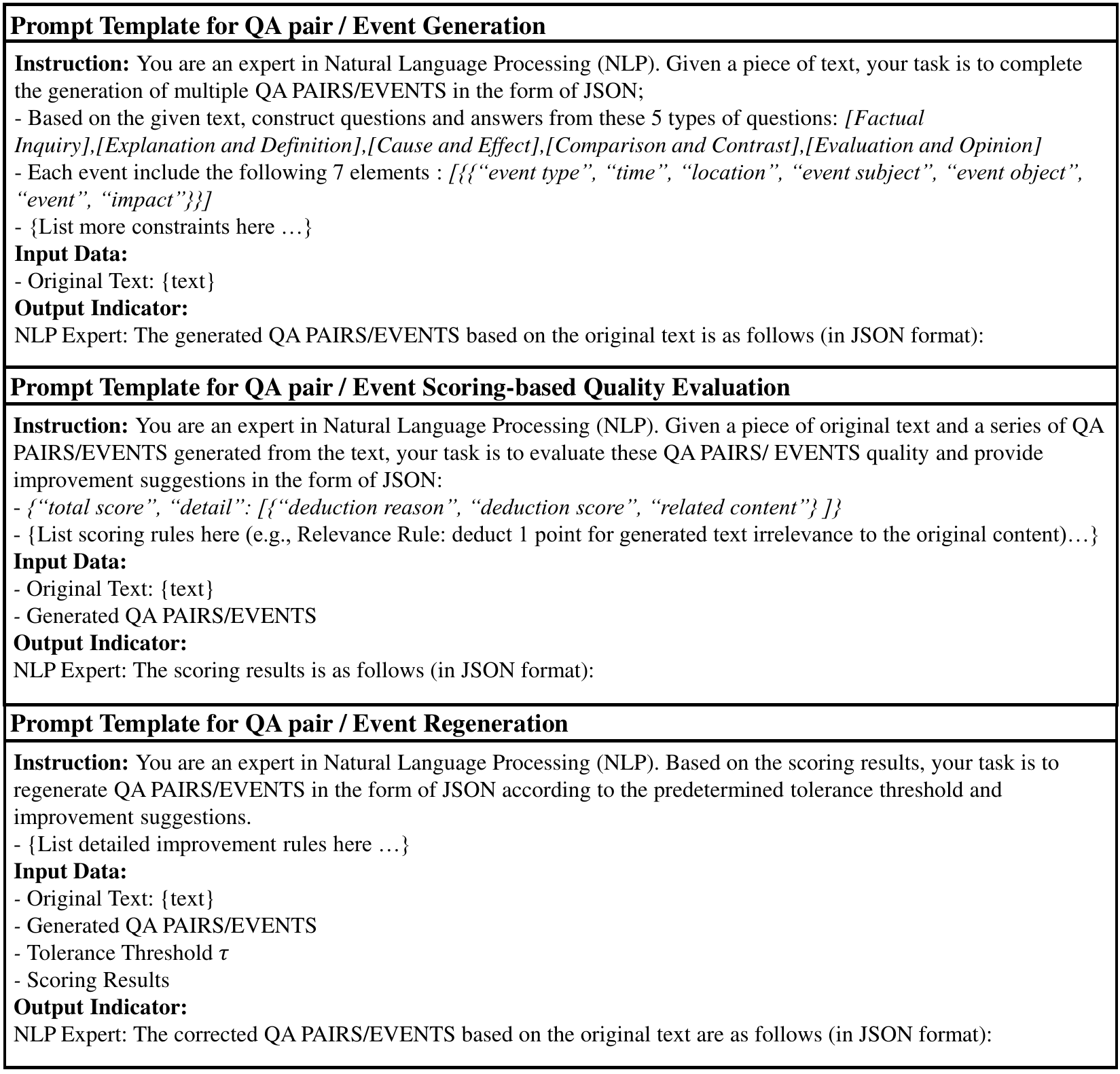}
  \caption{Prompt templates for QA pair/event generation, scoring-based quality evaluation and regeneration, respectively.}
  \label{prompt_template}
\end{figure}

\textit{Text Evaluation and Regeneration.}
In the open domain, evaluating the quality of QAG and EE is difficult due to the lack of labels. To address this, we introduce a robust prompt-based mechanism for evaluating the quality of generated texts and regenerating them if necessary, as outlined in Fig. \ref{prompt_template} and Algorithm \ref{alg1}. Using separate roles for generation and evaluation by different LLMs, \(\mathbf{LLM}_{\text{generator}}\) and \(\mathbf{LLM}_{\text{evaluator}}\), the double-check system scores and potentially rewrites outputs based on predefined criteria. As shown in the scoring-based quality evaluation prompt template in Fig. \ref{prompt_template}, each generated text starts with a perfect score of 10, and points are deducted for failures in \textit{Relevance}, \textit{Clarity}, \textit{Consistency}, and \textit{Completeness}. For example, the Relevance rule checks if each generated text faithfully reflects the original content, deducting one point for irrelevance. The scoring details are outlined in the JSON format \(\text{Score}_{\text{json}}\): \(\{\textit{``total score''}, \textit{``detail''}: [\{\textit{``deduction reason''}, \textit{``deduction score''}, \textit{``related content''}\}]\}\). Outputs that score at or below a set threshold \(\tau\) enter a regeneration prompting, where texts are adjusted or rewritten to ensure that outputs meet rigorous standards. This approach not only provides a reliable method for assessing generated texts but also controls the quality of the LLM outputs through a structured regeneration process based on scoring outcomes.

\textit{Unified Framework.} Algorithm \ref{alg1} details the QAEA module and shows that QAEA manages QA pairs and events in a similar manner. We conclude that QAEA is identified as a ``unified'' framework from two perspectives:
\begin{itemize}
\item{\textit{Text generation.} QAEA employs standardized prompt templates for generation, scoring-based quality evaluation, and regeneration, ensuring a uniform framework for both QAG and EE.}
\item{\textit{Text organization.} Additionally, all structured texts from both QA pairs and events are eventually converted into unstructured text, then embedded and incorporated into the vector database in a unified manner for dense retrieval.}
\end{itemize}

\begin{algorithm}[t]
\caption{QAEA}
\label{alg1} 
\begin{algorithmic}[1]
\REQUIRE A corpus $\mathcal{C} = \{t_1, t_2, \ldots, t_n\}$, large language model $\mathbf{LLMs}$, embedding model $\Phi$, score threshold $\tau$
\ENSURE Augmented vector database $\mathbf{VDB}_{\text{final}}$
\STATE Initialization: $\mathbf{VDB}_{\text{final}}, \mathbf{VDB}_{\text{QA}}, \mathbf{VDB}_{\text{EVENT}} \gets \emptyset$
\STATE $\mathbf{VDB}_{\text{ori}} \gets \Phi(\mathcal{C})$
\FOR{each $t_i \in \mathcal{C}$}
    \FOR{each type in $\{\text{``QA''}, \text{``EVENT''}\}$}
        \STATE $\text{type}_{\text{json}} \gets \mathbf{LLM}_{\text{generator}}(t_i, \text{type})$
        \STATE $\text{Score}_{\text{json}} \gets \mathbf{LLM}_{\text{evaluator}}(\text{type}_{\text{json}})$
        \IF{Score in $\text{Score}_{\text{json}} \leq \tau$}
            \STATE $\text{type}_{\text{json}} \gets \mathbf{LLM}_{\text{generator}}(t_i, \text{type}_{\text{json}}, \text{Score}_{\text{json}})$
        \ENDIF
        \STATE Unstructured form: $\{\text{element}_i^{(j)}\} \gets \text{type}_{\text{json}}$
        \IF{Texts Remain Independent}
            \STATE $\text{type}_i \gets \{\text{element}_i^{(j)}\}$
        \ELSE
            \STATE $\text{type}_i \gets \text{Concatenate}[\text{element}_i^{(j)}]$
        \ENDIF
        \STATE $V_{\text{type}_i} \gets \Phi(\text{type}_i)$
        \STATE $\mathbf{VDB}_{\text{type}} \gets \mathbf{VDB}_{\text{type}} \cup V_{\text{type}_i}$
    \ENDFOR
\ENDFOR
\STATE $\mathbf{VDB}_{\text{final}} \gets \mathbf{VDB}_{\text{ori}} + \mathbf{VDB}_{\text{QA}} + \mathbf{VDB}_{\text{EVENT}}$
\RETURN $\mathbf{VDB}_{\text{final}}$
\end{algorithmic}
\end{algorithm}

\subsection{Theory in QAEA-DR}
\label{sec:theory}
Now, we theoretically explain the effectiveness of QAEA-DR. In the following theoretical analysis, we consider only the case of Texts Remain Independent (TRI) mentioned in Section \ref{sec:overview}. It is evident that Texts Merge into One (TMO) can be viewed as a special case of TRI, and we will further discuss their differences in the experimental analysis. Given a text \( t_i \in \mathcal{C} \), we generate a text set \( \{t_i^{(j)}\} \) including QA pair texts and event texts, where \( j \) records the total number of final generated texts. In terms of vector representation, similarly, we combine the QA pair vectors \( V_{\text{QA}_i} \) and event vectors \( V_{\text{EVENT}_i} \) into \( V_{\text{GEN}_i} = \{ v^{(j)}_i \} \).

We first invoke the concept of \textit{fidelity} of the retrieval process and \textit{normalized margin} from previous work \cite{luan2021sparse}. Subsequently, without loss of generality, we demonstrate that these generated vectors either maintain or enhance the fidelity of the retrieval process. Theorem \ref{theorem: Text Augmentation} introduces the effectiveness of text augmentation for dense retrieval. Theorem \ref{theorem: QAEA} demonstrates the effectiveness of both QA Pair texts and event texts.

\textit{Fidelity} refers to the ability of dense vector models to maintain the distinctions made by traditional sparse bag-of-words retrieval models. Unlike sparse representations for exact matching, dense vector models map texts of arbitrary length into a fixed-length vector space, which may result in a loss of fidelity and consequently information loss. Importantly, our QAEA generates information-dense new texts that removes noisy texts and refines key information to improve fidelity. To measure fidelity, we introduce \textit{normalized margin} to indicate the distinction between the truly relevant text and other texts.

\begin{definition}[Normalized Margin]
Let \( v_q \), \( v_1 \), and \( v_2 \) be sentence embeddings in \( \mathbb{R}^d \). The normalized margin is defined as:
\begin{equation}
\mu(v_q, v_1, v_2) = \frac{\langle v_q, v_1 - v_2 \rangle}{\|v_q\| \cdot \|v_1 - v_2\|}
\end{equation}
\end{definition}

The normalized margin in retrieval models serves as a quantitative measure to evaluate fidelity and provides a comparative perspective on vector similarity. It indicates how distinctly a target text is separated from irrelevant ones in vector space, enhancing retrieval accuracy and relevance. Assuming \( v_1 \) is the target text vector, we expect a larger normalized margin between \( v_1 \) and \( v_2 \) (\(\mu(v_q, v_1, v_2) > 0\)), which indicates a greater difference in relevance between the target text and other texts for a given query.

Theorem \ref{theorem: Text Augmentation} uses normalized margin to demonstrate the effectiveness of text augmentation. 
The theorem holds under certain constraints, including \textit{relevance enhancement}, \textit{irrelevance consistency}, and \textit{orthogonality}. Under ideal text augmentation, the generated vectors of the target text should exhibit improved query relevance, while those of non-target texts should not be more relevant to the query than the original vectors. Additionally, in sparse retrieval, orthogonal vectors can be achieved by dividing the vocabulary into non-overlapping segments. Similarly, in dense models, we assume that vector representations of different texts are orthogonal when the content is irrelevant.

\begin{theorem}[Text Augmentation]
\label{theorem: Text Augmentation}
Given a text \( t_i \), let \( \{v_i^{(j)}\} \) represent a set of generated text vectors, where \( j \) records the total number of generated texts, and let \( v_i \) represent the original text vector.
Consider a text \( t_1 \) most relevant to a query \( q \) and a competing text \( t_2 \), we have generated text vector sets \( \{v_1^{(j)}\} \) and \( \{v_2^{(j)}\} \), respectively.  There exists at least one generated vector \( v_1^{(0)} \in \{v_1^{(j)}\} \) such that the normalized margins
\begin{equation}
\mu(v_q, v_1^{(0)}, v_2) \geq \mu(v_q, v_1, v_2) \tag{2}
\label{eq:TRI1}
\end{equation}
and
\begin{equation}
\mu(v_q, v_1^{(0)}, v_2^{(j)}) \geq \mu(v_q, v_1, v_2), \forall j \tag{3}
\label{eq:TRI2}
\end{equation}
are both satisfied under the following conditions:

(i) Relevance Enhancement: \(\exists v_1^{(0)} \in \{v_1^{(j)}\} \), s.t. \( \langle v_q, v_1^{(j)} \rangle \geq \langle v_q, v_1 \rangle \); 

(ii) Irrelevance Consistency: \(\forall j\), \( \langle v_q, v_2^{(j)} \rangle \leq \langle v_q, v_2 \rangle \); 

(iii) Orthogonality: Any two generated text vectors across \( \{v_1^{(j)}\} \) and \( \{v_2^{(j)}\} \) are orthogonal to each other. Additionally, different vector segments derived from a text are orthogonal to each other.

\end{theorem}

\begin{proof}
First, we construct hypothetical orthogonal noise vectors \( v_{noise_1} \) for \( v_1 \) and \( v_{noise_2} \) for \( v_2 \). Representing \( v_1 \) and \( v_2 \) as:
\begin{align}
v_1 &= v_1^{(0)} + v_{noise_1}, \tag{4} \\
v_2 &= v_2^{(j)} + v_{noise_2}, \forall j \tag{5}
\end{align}
Then, the squared norm expansion for the difference between \( v_1 \) and \( v_2 \) gives us:
\begin{align}
\|v_1 - v_2\|^2 &= \|v_1^{(0)} - v_2^{(j)} + v_{noise_1} - v_{noise_2}\|^2 \nonumber \\
&= \|v_1^{(0)} - v_2^{(j)}\|^2 + \|v_{noise_1} - v_{noise_2}\|^2 \nonumber \\
&\quad + 2\langle v_1^{(0)} - v_2^{(j)}, v_{noise_1} - v_{noise_2} \rangle. \tag{6}
\end{align}
Given the orthogonality condition (iii), the cross term vanishes:
\begin{align}
\langle v_1^{(0)} - v_2^{(j)}, v_{noise_1} - v_{noise_2} \rangle = 0, \tag{7}
\end{align}
thus,
\begin{align}
\|v_1 - v_2\|^2 &= \|v_1^{(0)} - v_2^{(j)}\|^2 + \|v_{noise_1} - v_{noise_2}\|^2 \nonumber \\
&\geq \|v_1^{(0)} - v_2^{(j)}\|^2. \tag{8}
\label{ineq:8}
\end{align}
Similarly, we have:
\begin{align}
 \|v_1 - v_2\|^2 &\geq \|v_1^{(0)} - v_2 \|^2. \tag{9}
 \label{ineq:9}
\end{align}

For the numerator of normalized margin, we have:
\begin{align}
\langle v_q, v_1^{(0)} - v_2 \rangle \geq \langle v_q, v_1 \rangle - \langle v_q, v_2 \rangle \tag{10}
\label{ineq:10}
\end{align}
\begin{align}
\langle v_q, v_1^{(0)} - v_2^{(j)} \rangle \geq \langle v_q, v_1 \rangle - \langle v_q, v_2 \rangle \tag{11}
\label{ineq:11}
\end{align}
based on the conditions (i) and (ii).

By combining inequalities \ref{ineq:9} and \ref{ineq:10}, we conclude that:
\begin{align}
\mu(v_q, v_1^{(0)}, v_2) \geq \mu(v_q, v_1, v_2) \tag{12}
\end{align}
Similarly, combining inequalities \ref{ineq:8} and \ref{ineq:11}, we conclude that:
\begin{align}
\mu(v_q, v_1^{(0)}, v_2^{(j)}) \geq \mu(v_q, v_1, v_2), \forall j \tag{13}
\end{align}
Therefore, the formulas (\ref{eq:TRI1}) and (\ref{eq:TRI2}) hold, proving that the text augmentation approach maintains or improves retrieval fidelity.
\end{proof}

The constraints in Theorem \ref{theorem: Text Augmentation} are based on the assumption that each generated text is of high quality and contains a portion of the original text's information. Given the text \( t_1 \) and the related query, the ideal generated vector \( v_1^{(0)} \) enhances retrieval fidelity by reducing noise and condensing query-relevant information from the text.

Building on the demonstrated effectiveness of text augmentation, we show that incorporating both QA pair vectors and event vectors into the text augmentation framework enhances retrieval fidelity more effectively than using a single type of generated text. Within the constraints of relevance enhancement and irrelevance consistency, the introduction of new high-quality generated vectors will only improve fidelity, making this conclusion clear.

\begin{theorem}[QAEA]
\label{theorem: QAEA}
Given a text \( t_1 \) most relevant to a query \( q \) and a competing text \( t_2 \), we generate sets of QA pair vectors \( \{v_{qa_1}^{(j)}\} \) and event vectors \( \{v_{event_1}^{(j)}\} \) for \( t_1 \), and \( \{v_{qa_2}^{(j)}\} \) and \( \{v_{event_2}^{(j)}\} \) for \( t_2 \), respectively, where \( j \) records the number of generated texts. Let \( v_1 \) and \( v_2 \) represent the original text vectors of \( t_1 \) and \( t_2 \). Given \( \{v_1^{(j)}\} = \{v_{qa_1}^{(j)}\} \cup \{v_{event_1}^{(j)}\} \), there exists at least one generated vector \( v_1^{(0)} \in \{v_1^{(j)}\} \) such that the normalized margins
\begin{equation}
\mu(v_q, v_1^{(0)}, v_2) \geq \mu(v_q, v_{qa_1}^{(j)}, v_2), \forall j \tag{14}
\label{eq:QA1}
\end{equation}
and
\begin{equation}
\mu(v_q, v_1^{(0)}, v_2) \geq \mu(v_q, v_{event_1}^{(j)}, v_2), \forall j \tag{15}
\label{eq:Event1}
\end{equation}
are satisfied under the following conditions:

(i) Relevance Enhancement: \(\exists j\), s.t. \( \langle v_q, v_{qa_1}^{(j)} \rangle \geq \langle v_q, v_1 \rangle \) or \( \langle v_q, v_{event_1}^{(j)} \rangle \geq \langle v_q, v_1 \rangle \);

(ii) Irrelevance Consistency: \(\forall j\), \( \langle v_q, v_{qa_2}^{(j)} \rangle \leq \langle v_q, v_2 \rangle \) and \( \langle v_q, v_{event_2}^{(j)} \rangle \leq \langle v_q, v_2 \rangle \);

(iii) Orthogonality: Any two generated text vectors across \( \{v_{qa_1}^{(j)}\} \) and \( \{v_{event_1}^{(j)}\} \) are orthogonal to each other.
\end{theorem}

\begin{proof}
We start by noting that the set \( \{v_1^{(j)}\} = \{v_{qa_1}^{(j)}\} \cup \{v_{event_1}^{(j)}\} \). From the relevance enhancement condition, there exists \( v_1^{(0)} \in \{v_{qa_1}^{(j)}\} \cup \{v_{event_1}^{(j)}\} \) that maximizes the inner product with \( v_q \). The remaining proof follows similarly to the proof of Theorem \ref{theorem: Text Augmentation} and can be easily derived.
\end{proof}

Theorem \ref{theorem: QAEA} proposes that integrating more generated text vector representations is more likely to increase retrieval fidelity compared to using only one type of generated text vector representation.

\begin{table*}[]
\centering
\caption{The NDCG (\(\times 100\)) comparisons of baselines and our QAEA-DR on four datasets}
\label{tab:main_result}
\resizebox{\textwidth}{!}{%
\begin{tabular}{l|c l|l l|l l|l l|c l|l l|l l|l l}
\hline
\multirow{3}{*}{} & \multicolumn{8}{c|}{sCNPR} & \multicolumn{8}{c}{T2Retrieval} \\
\cline{2-17}
                  & \multicolumn{4}{c|}{m3e} & \multicolumn{4}{c|}{bge-zh} & \multicolumn{4}{c|}{m3e} & \multicolumn{4}{c}{bge-zh} \\
\cline{2-17}
                  & \multicolumn{2}{c|}{NDCG@1} & \multicolumn{2}{c|}{NDCG@10} & \multicolumn{2}{c|}{NDCG@1} & \multicolumn{2}{c|}{NDCG@10} & \multicolumn{2}{c|}{NDCG@1} & \multicolumn{2}{c|}{NDCG@10} & \multicolumn{2}{c|}{NDCG@1} & \multicolumn{2}{c}{NDCG@10} \\
\hline
Baseline          & \multicolumn{2}{c|}{70.23} & \multicolumn{2}{c|}{79.31} & \multicolumn{2}{c|}{73.17} & \multicolumn{2}{c|}{81.73} & \multicolumn{2}{c|}{78.33} & \multicolumn{2}{c|}{86.59} & \multicolumn{2}{c|}{82.19} & \multicolumn{2}{c}{89.21} \\
\hline
                  & GPT & GLM & GPT & GLM & GPT & GLM & GPT & GLM & GPT & GLM & GPT & GLM & GPT & GLM & GPT & GLM \\
QAEA (TRI)        & \textbf{77.92} & \textbf{77.92} & \textbf{85.03} & 84.99 & \textbf{83.15} & 81.83 & \textbf{88.96} & 88.22 & 74.60 & 74.77 & 83.65 & 83.58 & 80.78 & 80.52 & 88.51 & 88.34 \\
QAEA (TMO)        & 72.84 & 72.71 & 81.25 & 80.92 & 75.74 & 75.86 & 83.38 & 83.27 & 79.13 & \textbf{79.23} & \textbf{87.24} & 87.19 & 82.74 & \textbf{83.09} & 89.61 & \textbf{89.74} \\
\hline
\hline
\multirow{3}{*}{} & \multicolumn{8}{c|}{HotpotQA} & \multicolumn{8}{c}{MS MARCO} \\
\cline{2-17}
                  & \multicolumn{4}{c|}{dpr} & \multicolumn{4}{c|}{bge-en} & \multicolumn{4}{c|}{dpr} & \multicolumn{4}{c}{bge-en} \\
\cline{2-17}
                  & \multicolumn{2}{c|}{NDCG@1} & \multicolumn{2}{c|}{NDCG@10} & \multicolumn{2}{c|}{NDCG@1} & \multicolumn{2}{c|}{NDCG@10} & \multicolumn{2}{c|}{NDCG@1} & \multicolumn{2}{c|}{NDCG@10} & \multicolumn{2}{c|}{NDCG@1} & \multicolumn{2}{c}{NDCG@10} \\
\hline
Baseline          & \multicolumn{2}{c|}{79.38} & \multicolumn{2}{c|}{74.24} & \multicolumn{2}{c|}{95.79} & \multicolumn{2}{c|}{92.70} & \multicolumn{2}{c|}{70.23} & \multicolumn{2}{c|}{80.64} & \multicolumn{2}{c|}{95.93} & \multicolumn{2}{c}{98.02} \\
\hline
                  & GPT & GLM & GPT & GLM & GPT & GLM & GPT & GLM & GPT & GLM & GPT & GLM & GPT & GLM & GPT & GLM \\
QAEA (TRI)        & \textbf{85.26} & 85.17 & \textbf{78.92} & 78.47 & \textbf{96.79} & 96.72 & \textbf{93.61} & 93.49 & \textbf{78.20} & 78.09 & \textbf{86.95} & 86.73 & 94.93 & 94.90 & 97.64 & 97.60 \\
QAEA (TMO)        & 81.97 & 81.66 & 76.01 & 75.96 & 96.06 & 95.99 & 92.91 & 92.83 & 74.33 & 74.96 & 83.80 & 83.72 & \textbf{96.21} & 96.17 & \textbf{98.19} & 98.15 \\
\hline
\end{tabular}%
}
\\[2pt]
\raggedright
The values in bold represent the largest NDCG.
\end{table*}

\section{Experiment}
\subsection{Datasets and Baselines}
\textit{Datasets.} We utilize four passage retrieval datasets to evaluate our QAEA-DR. Due to the high computational cost of multiple LLM tasks, we used a subset of the complete open datasets for our experiments.
To ensure the unbiased selection of dataset subsets, we applied the random.sample() method in Python to randomly select texts from the corpus.

\begin{itemize}
\item{\textit{sCNPR:} a proprietary small-scale Chinese News Passage Retrieval dataset (sCNPR) that we created from real-world news articles and user queries. sCNPR covers diverse topics including economic, social, political, scientific, and entertainment news. sCNPR contains 1083 news texts and 2382 user queries, with an average text length of 655 words.} 

\item{\textit{T2Retrieval} \cite{xie2023t2ranking}: an open Chinese passage retrieval dataset with 4-level relevance labels, useful for assessing models' ability to distinguish fine-grained relevance. We sampled 5000 instances from training set, with an average text length of 438 words.}

\item{\textit{MS MARCO} \cite{nguyen2016ms}: an open English passage retrieval dataset from Bing, featuring single relevance labels. We sampled 5000 instances from training set, with an average text length of 58 words.}

\item{\textit{HotpotQA} \cite{yang2018hotpotqa}: an open English passage retrieval dataset from Wikipedia, featuring single relevance labels. We sampled 5000 instances from training set, with an average text length of 67 words.}

\end{itemize}

\textit{Baselines.} 
We select four sentence embedding models that do not use text augmentation and embed only the original texts into a vector database (\(\text{VDB}_\text{ori}\)) as our baselines for dense retrieval. We evaluate our QAEA-DR method, which applies text augmentation for augmented vector database (\(\text{VDB}_\text{final}\)), against these baselines to demonstrate its performance improvement under the same standard retrieval process. We utilize the Chinese sentence embedding models m3e-base and bge-large-zh-v1.5 \cite{xiao2023c}. For English, we employ the embedding models dpr-ctx\textunderscore{}encoder-multiset-base \cite{karpukhin2020dense} and bge-large-en-v1.5. These sentence embedding models represent effective and illustrative examples of sentence embedding techniques.

Besides, our LLM-based QAG method is compared with the currently popular LMQG model \cite{ushio2023empirical} to demonstrate the effectiveness of our approach. In our experiment, we employ the End2end T5-large-squad-QAG model of LMQG for comparison. On the other hand, there are no effective open-domain multi-event extraction models available for comparison.

\subsection{Implementation Details}
\textit{LLM Implementation.} We primarily employ two LLMs as text augmentation generators and another LLM as a evaluator for evaluating the quality of the generated texts. In our experiments, we utilize API calls to GPT-3.5-turbo and locally deployed ChatGLM3-6B \cite{du2022glm} as the two generators for QAG and EE. Additionally, in Section \ref{More LLMs Generation Results}, we deployed Mistral-7B \cite{jiang2023mistral} and Llama-3-8B \cite{touvron2023llama} to compare the performance of different open-source LLMs. We set the temperature parameter to 0 to ensure the generated text is precise and faithful to the original input. In our experiments, since ChatGLM3-6B is not friendly with JSON output, we manually create 20 JSON samples for fine-tuning to ensure the stable JSON output. Additionally, we select DeepSeek-v2\footnote{https://platform.deepseek.com} as the evaluator for scoring the quality of generated texts, using API calls for text generation. Given that DeepSeek-v2 significantly outperforms GPT-3.5-turbo and ChatGLM3-6B on LLM evaluation benchmarks, it serves as an effective independent evaluator to enhance the reliability of the output assessments. All experiments run on the NVIDIA RTX 3090 GPU (24 GB).

\textit{Experimental Variables.}
In addition to using different LLMs, we analyze the impact of various embedding models \(\Phi\), thresholds \(\tau\) (default set at 9 out of 10), dataset size, and text organization strategies (Texts Remain Independent (TRI) and Texts Merge into One (TMO)) on QAEA-DR. We also evaluate the separate effects of QA pairs and events.

\subsection{Evaluation Metrics}
We employ Normalized Discounted Cumulative Gain (NDCG) as the critical metric within the MTEB \cite{muennighoff2023mteb} evaluation framework to assess the retrieval performance. Besides, MTEB also records metrics like mean reciprocal rank (MRR), recall, precision, etc. NDCG measures retrieval ranking quality by accounting for the positions of relevant documents, providing a more comprehensive evaluation compared to other metrics.

\subsection{Main Results}
In this section, we present the main experimental results comparing the performance of QAEA-DR and the baselines across four datasets based on NDCG@1 and NDCG@10 metrics. Table \ref{tab:main_result} displays the QAEA-DR results from a complete three-step prompting process.

\begin{figure}[ht]
  \centering
  \includegraphics[width=1\linewidth]{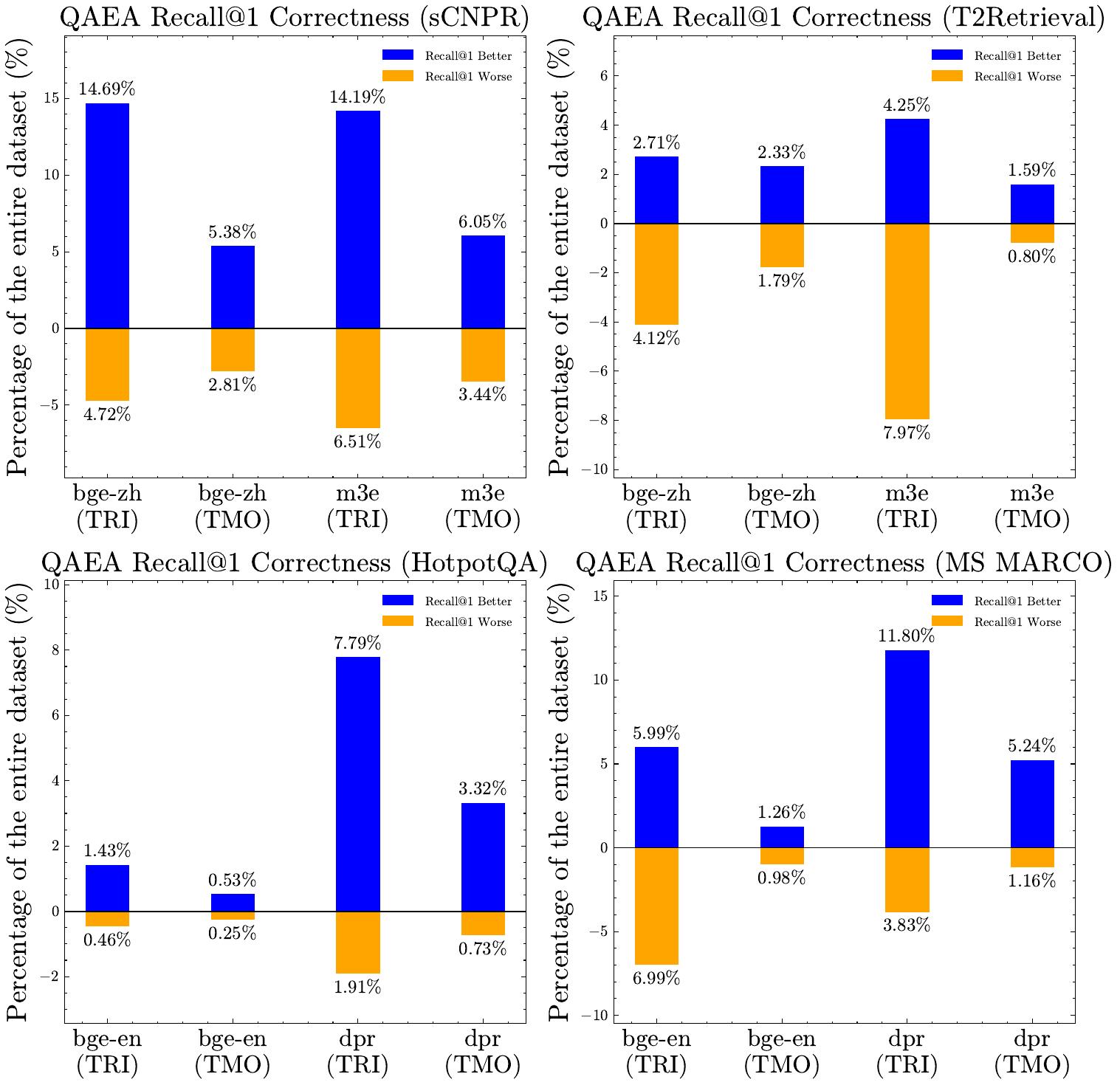}
  \caption{QAEA-DR \textbf{vs} Baseline on Recall@1. The blue bar represents the percentage of the entire dataset where
QAEA-DR correctly recalls at rank 1, while the baseline does not; the orange bar indicates the opposite. The difference between the blue and orange bars quantifies the actual improvement that QAEA-DR provides over the baseline.}
  \label{worsebetter_analysis}
\end{figure}

\textit{Results Across Datasets.}
Overall, our QAEA-DR outperforms the baselines across four datasets, achieving optimal results as shown in Table \ref{tab:main_result}. Our method exhibits effectiveness and robustness across datasets with varying languages, text sources, and text lengths. Notably, on the sCNPR dataset, which consists of relatively long news texts, QAEA-DR shows the most significant improvement, with the average NDCG@1 score rising from 71.70\% to 80.54\%. The results from the sCNPR dataset are significant since sCNPR consists of long text data and has not been used for embedding model training. On the HotpotQA and MS MARCO datasets, which primarily consist of short texts, there is an average increase of approximately 4 points in NDCG@1. Although the T2Retrieval dataset presents a challenging task with its fine-grained scoring metrics, QAEA-DR still manages to achieve an average increase of nearly one point in NDCG@1.

\textit{Results Across Embedding Models.} Table \ref{tab:main_result} demonstrates that QAEA-DR improves retrieval performance across four embedding models. It is worth noting that the initial retrieval performance of the embedding models affects the extent of improvement achieved by QAEA-DR. For instance, in English embedding models, QAEA-DR increases the average NDCG@1 score by 6.93\% on the traditional dpr-ctx\textunderscore{}encoder-multiset-base (dpr) model and by 0.64\% on the more advanced bge-large-en-v1.5 (bge-en) model. In Chinese embedding scenarios, both the m3e-base (m3e) and bge-large-zh-v1.5 (bge-zh) models show comparable retrieval capabilities, and QAEA-DR delivers similar enhancements to each. Therefore, QAEA-DR not only significantly enhances embedding models with weaker retrieval capabilities but also consistently improves those already performing well.

\begin{table}[h]
\centering
\caption{Comparisons of LLM Single-Generation and with Regeneration Performance in QAEA (TRI): NDCG@1 (\(\times 100\))}
\label{tab:compare_TRI}
\setlength{\tabcolsep}{2.5pt}
\begin{tabular}{l|ll|ll|ll|ll|}
\hline
\multirow{2}{*}{} & \multicolumn{4}{c|}{sCNPR}                                 & \multicolumn{4}{c}{T2Retrieval}                                         \\ \cline{2-9} 
                  & \multicolumn{2}{c|}{m3e}        & \multicolumn{2}{c|}{bge-zh}     & \multicolumn{2}{c|}{m3e}         & \multicolumn{2}{c}{bge-zh}      \\ \hline
LLM               & GPT           & GLM            & GPT            & GLM            & GPT            & GLM            & GPT            & GLM            \\
Single-generation     & 77.34          & 67.88          & 83.04          & 74.90          & 74.01         & 74.03          & 80.28          & 80.33          \\
+Regeneration   & \textbf{77.92} & \textbf{77.92} & \textbf{83.15} & \textbf{81.83} & \textbf{74.60} & \textbf{74.77} & \textbf{80.78} & \textbf{80.52} \\ \hline
\hline
\multirow{2}{*}{} & \multicolumn{4}{c|}{HotpotQA}                                    & \multicolumn{4}{c}{MS MARCO}                                      \\ \cline{2-9} 
                  & \multicolumn{2}{c|}{dpr}        & \multicolumn{2}{c|}{bge-en}     & \multicolumn{2}{c|}{dpr}         & \multicolumn{2}{c}{bge-en}      \\ \hline
LLM               & GPT           & GLM            & GPT            & GLM            & GPT            & GLM            & GPT            & GLM            \\
Single-generation     & 84.83          & 83.82          & 96.61          & 96.45          & 77.28         & 77.88          & 94.80           & 94.45          \\
+Regeneration   & \textbf{85.26} & \textbf{85.17} & \textbf{96.79} & \textbf{96.72} & \textbf{78.20} & \textbf{78.09} & \textbf{94.93} & \textbf{94.90} \\ \hline
\end{tabular}
\raggedright
\\[2pt]
The values in bold represent the largest NDCG@1.
\end{table}

\begin{table}[h]
\centering
\caption{Comparisons of LLM Single-Generation and with Regeneration Performance in QAEA (TMO): NDCG@1 (\(\times 100\))}
\label{tab:compare_TMO}
\setlength{\tabcolsep}{2.5pt}
\begin{tabular}{l|ll|ll|ll|ll}
\hline
\multirow{2}{*}{} & \multicolumn{4}{c|}{sCNPR}                                  & \multicolumn{4}{c}{T2Retrieval}                                         \\ \cline{2-9}
                  & \multicolumn{2}{c|}{m3e}         & \multicolumn{2}{c|}{bge-zh}     & \multicolumn{2}{c|}{m3e}         & \multicolumn{2}{c}{bge-zh}      \\ \hline
LLM               & GPT            & GLM            & GPT            & GLM            & GPT            & GLM            & GPT            & GLM            \\
Single-generation & 72.38          & 71.45          & 75.25          & 74.01          & 79.07          & 79.07          & 82.33          & 82.53          \\
+Regeneration     & \textbf{72.84} & \textbf{72.71} & \textbf{75.74} & \textbf{75.86} & \textbf{79.13} & \textbf{79.23} & \textbf{82.74} & \textbf{83.09} \\ \hline
\hline
\multirow{2}{*}{} & \multicolumn{4}{c|}{HotpotQA}                                     & \multicolumn{4}{c}{MS MARCO}                                      \\ \cline{2-9}
                  & \multicolumn{2}{c|}{dpr}         & \multicolumn{2}{c|}{bge-en}     & \multicolumn{2}{c|}{dpr}         & \multicolumn{2}{c}{bge-en}      \\ \hline
LLM               & GPT            & GLM            & GPT            & GLM            & GPT            & GLM            & GPT            & GLM            \\
Single-generation & 81.82          & 80.66          & 95.91          & 95.87          & 73.51          & 74.15          & 96.10           & 96.04          \\
+Regeneration     & \textbf{81.97} & \textbf{81.66} & \textbf{96.06} & \textbf{95.99} & \textbf{74.33} & \textbf{74.96} & \textbf{96.21} & \textbf{96.17} \\ \hline
\end{tabular}
\raggedright
\\[2pt]
The values in bold represent the largest NDCG@1.
\end{table}

\textit{Impact of Text Organization.}
We also observe obviously different retrieval performances between the two final generated text forms: Texts Remain Independent (TRI) and Texts Merge into One (TMO), as described in Section \ref{sec:overview}. Overall, TRI can be considered a form of text decomposition, whereas TMO involves reassembling the decomposed texts. 
From Table \ref{tab:main_result}, we observe the following:

\begin{table*}[h]
\centering
\caption{The NDCG (\(\times 100\)) comparisons of Ablation Study on QAEA (TMO)}
\label{tab:ablation}
\resizebox{\textwidth}{!}{%
\begin{tabular}{l|llllllll|llllllll}
\hline
\multirow{3}{*}{} & \multicolumn{8}{c|}{sCNPR} & \multicolumn{8}{c}{T2Retrieval} \\ \cline{2-17} 
 & \multicolumn{4}{c|}{m3e} & \multicolumn{4}{c|}{bge-zh} & \multicolumn{4}{c|}{m3e} & \multicolumn{4}{c}{bge-zh} \\ \cline{2-17} 
 & \multicolumn{2}{c|}{NDCG@1} & \multicolumn{2}{c|}{NDCG@10} & \multicolumn{2}{c|}{NDCG@1} & \multicolumn{2}{c|}{NDCG@10} & \multicolumn{2}{c|}{NDCG@1} & \multicolumn{2}{c|}{NDCG@10} & \multicolumn{2}{c|}{NDCG@1} & \multicolumn{2}{c}{NDCG@10} \\ \hline
Baseline (Original) & \multicolumn{2}{c|}{70.23} & \multicolumn{2}{c|}{79.31} & \multicolumn{2}{c|}{73.17} & \multicolumn{2}{c|}{81.73} & \multicolumn{2}{c|}{78.33} & \multicolumn{2}{c|}{86.59} & \multicolumn{2}{c|}{82.19} & \multicolumn{2}{c}{89.21} \\ \hline
 & GPT & \multicolumn{1}{l|}{GLM} & GPT & \multicolumn{1}{l|}{GLM} & GPT & \multicolumn{1}{l|}{GLM} & GPT & GLM & GPT & \multicolumn{1}{l|}{GLM} & GPT & \multicolumn{1}{l|}{GLM} & GPT & \multicolumn{1}{l|}{GLM} & GPT & GLM \\
QA & 68.85 & \multicolumn{1}{l|}{67.38} & 77.67 & \multicolumn{1}{l|}{76.44} & 72.12 & \multicolumn{1}{l|}{71.41} & 80.26 & 79.74 & 74.79 & \multicolumn{1}{l|}{73.89} & 83.48 & \multicolumn{1}{l|}{82.99} & 78.61 & \multicolumn{1}{l|}{78.32} & 85.95 & 85.81 \\
Event & 68.30 & \multicolumn{1}{l|}{68.43} & 77.29 & \multicolumn{1}{l|}{77.43} & 71.41 & \multicolumn{1}{l|}{71.50} & 79.74 & 79.60 & 70.55 & \multicolumn{1}{l|}{70.49} & 80.47 & \multicolumn{1}{l|}{80.62} & 76.67 & \multicolumn{1}{l|}{76.38} & 84.69 & 84.61 \\
Original+QA & 72.54 & \multicolumn{1}{l|}{72.04} & 81.07 & \multicolumn{1}{l|}{80.47} & 76.07 & \multicolumn{1}{l|}{75.90} & 83.51 & 83.25 & 78.73 & \multicolumn{1}{l|}{78.35} & 87.04 & \multicolumn{1}{l|}{86.71} & 82.85 & \multicolumn{1}{l|}{82.69} & 89.62 & 89.64 \\
Original+Event & 72.51 & \multicolumn{1}{l|}{72.38} & 81.01 & \multicolumn{1}{l|}{80.75} & 74.48 & \multicolumn{1}{l|}{74.85} & 82.59 & 82.67 & 77.24 & \multicolumn{1}{l|}{77.39} & 86.01 & \multicolumn{1}{l|}{86.08} & 82.26 & \multicolumn{1}{l|}{82.19} & 89.26 & 89.21 \\
QA+Event & 69.02 & \multicolumn{1}{l|}{69.02} & 77.80 & \multicolumn{1}{l|}{77.43} & 72.08 & \multicolumn{1}{l|}{72.25} & 80.31 & 80.05 & 75.72 & \multicolumn{1}{l|}{75.79} & 84.52 & \multicolumn{1}{l|}{84.37} & 80.00 & \multicolumn{1}{l|}{80.45} & 87.18 & 87.63 \\
Original+QA+Event   (QAEA) & \textbf{72.84} & \multicolumn{1}{l|}{72.71} & \textbf{81.25} & \multicolumn{1}{l|}{80.92} & 75.74 & \multicolumn{1}{l|}{\textbf{75.86}} & \textbf{83.38} & 83.27 & 79.13 & \multicolumn{1}{l|}{\textbf{79.23}} & \textbf{87.24} & \multicolumn{1}{l|}{87.19} & 82.74 & \multicolumn{1}{l|}{\textbf{83.09}} & 89.61 & \textbf{89.74} \\ \hline
\hline
\multirow{3}{*}{} & \multicolumn{8}{c|}{HotpotQA} & \multicolumn{8}{c}{MS MARCO} \\ \cline{2-17} 
 & \multicolumn{4}{c|}{dpr} & \multicolumn{4}{c|}{bge-en} & \multicolumn{4}{c|}{dpr} & \multicolumn{4}{c}{bge-en} \\ \cline{2-17} 
 & \multicolumn{2}{c|}{NDCG@1} & \multicolumn{2}{c|}{NDCG@10} & \multicolumn{2}{c|}{NDCG@1} & \multicolumn{2}{c|}{NDCG@10} & \multicolumn{2}{c|}{NDCG@1} & \multicolumn{2}{c|}{NDCG@10} & \multicolumn{2}{c|}{NDCG@1} & \multicolumn{2}{c}{NDCG@10} \\ \hline
Baseline (Original) & \multicolumn{2}{c|}{79.38} & \multicolumn{2}{c|}{74.24} & \multicolumn{2}{c|}{95.79} & \multicolumn{2}{c|}{92.70} & \multicolumn{2}{c|}{70.23} & \multicolumn{2}{c|}{80.64} & \multicolumn{2}{c|}{95.93} & \multicolumn{2}{c}{98.02} \\ \hline
 & GPT & \multicolumn{1}{l|}{GLM} & GPT & \multicolumn{1}{l|}{GLM} & GPT & \multicolumn{1}{l|}{GLM} & GPT & GLM & GPT & \multicolumn{1}{l|}{GLM} & GPT & \multicolumn{1}{l|}{GLM} & GPT & \multicolumn{1}{l|}{GLM} & GPT & GLM \\
QA & 79.07 & \multicolumn{1}{l|}{79.27} & 72.12 & \multicolumn{1}{l|}{72.10} & 95.33 & \multicolumn{1}{l|}{95.33} & 90.76 & 90.44 & 71.67 & \multicolumn{1}{l|}{71.78} & 81.85 & \multicolumn{1}{l|}{81.38} & 93.98 & \multicolumn{1}{l|}{93.28} & 96.98 & 96.53 \\
Event & 73.01 & \multicolumn{1}{l|}{72.24} & 67.80 & \multicolumn{1}{l|}{67.18} & 94.44 & \multicolumn{1}{l|}{94.40} & 89.63 & 88.83 & 65.29 & \multicolumn{1}{l|}{64.80} & 76.56 & \multicolumn{1}{l|}{76.10} & 91.56 & \multicolumn{1}{l|}{90.86} & 95.36 & 94.93 \\
Original+QA & 82.47 & \multicolumn{1}{l|}{\textbf{82.55}} & 76.28 & \multicolumn{1}{l|}{\textbf{76.38}} & 95.91 & \multicolumn{1}{l|}{95.91} & 92.83 & 92.85 & 74.50 & \multicolumn{1}{l|}{74.48} & \textbf{84.03} & \multicolumn{1}{l|}{84.01} & 96.08 & \multicolumn{1}{l|}{96.12} & 98.13 & 98.14 \\
Original+Event & 80.81 & \multicolumn{1}{l|}{80.50} & 75.32 & \multicolumn{1}{l|}{75.27} & 96.10 & \multicolumn{1}{l|}{96.06} & \textbf{92.95} & 92.89 & 72.44 & \multicolumn{1}{l|}{73.19} & 82.81 & \multicolumn{1}{l|}{83.05} & 95.85 & \multicolumn{1}{l|}{95.89} & 98.02 & 98.03 \\
QA+Event & 78.69 & \multicolumn{1}{l|}{79.15} & 72.28 & \multicolumn{1}{l|}{72.02} & 95.56 & \multicolumn{1}{l|}{95.60} & 91.44 & 91.37 & 72.33 & \multicolumn{1}{l|}{72.10} & 82.31 & \multicolumn{1}{l|}{82.02} & 94.56 & \multicolumn{1}{l|}{94.67} & 97.29 & 97.35 \\
Original+QA+Event   (QAEA) & 81.97 & \multicolumn{1}{l|}{81.66} & 76.01 & \multicolumn{1}{l|}{75.96} & \textbf{96.06} & \multicolumn{1}{l|}{95.99} & 92.91 & 92.83 & 74.33 & \multicolumn{1}{l|}{\textbf{74.96}} & 83.80 & \multicolumn{1}{l|}{83.72} & \textbf{96.21} & \multicolumn{1}{l|}{96.17} & \textbf{98.19} & 98.15 \\ \hline
\end{tabular}%
}
\\[2pt]
\raggedright
The values in bold represent the largest NDCG.
\end{table*}

\begin{itemize}
\item{QAEA (TRI) achieves higher peak results on the sCNPR, HotpotQA, and MS MARCO datasets. It records an average NDCG@1 score of 86.04\% across these datasets, exceeding the score of 82.98\% recorded by TMO. As outlined in Section \ref{sec:theory}, QAEA satisfies Relevance Enhancement and Irrelevance Consistency under ideal conditions. However, in practical settings such as T2Retrieval, where texts are lengthy and filled with irrelevant content, TRI may increase the risk of irrelevance mismatching due to retaining unclear, irrelevant generated texts.}
\item{QAEA (TMO) demonstrates consistent stability, outperforming baselines across all datasets. For example, in the T2Retrieval dataset, TMO enhances the density of key information and reduces the density of noise by reassembling generated texts, resulting in better results than TRI.}
\item{
From Fig. \ref{worsebetter_analysis}, QAEA with either TMO or TRI overall improves recall@1 performance compared to the baselines, as indicated by the longer blue bars than orange. TMO consistently exhibits fewer instances of reduced recall@1 performance compared to TRI (shorter orange bars) across different datasets and embedding models, indicating more stable performance. In datasets like T2Retrieval with some typos, redundancies, or irrelevant texts, TRI may extract noisy texts, leading to an increase in poor performance cases; TMO merges generated texts, minimizing the impact of noisy text and resulting in fewer recall@1 worse cases.}
\end{itemize}

\textit{Impact of LLM Generation.}
To demonstrate the robustness of text enhancement methods across different LLMs, we conducted the experimental analysis using two models: ChatGLM3-6B (GLM) as a smaller-scale model and GPT-3.5-turbo (GPT) as a medium-scale model. From Table \ref{tab:main_result} we observe that GPT exhibits better performance than GLM, as expected. However, the average NDCG@1 difference between GPT and GLM is only 0.07\%. Table \ref{tab:compare_TRI} and Table \ref{tab:compare_TMO} illustrate that the performance gaps between the two LLMs increase from 0.07\% to 1.1\% if the regeneration mechanism is absent. This demonstrates the robustness of the complete three-step prompting method under different LLMs as generators. Specifically, Table \ref{tab:compare_TRI} compares the retrieval performance of QAEA (TRI) in two scenarios: (1) using single generation (Single-generation) and (2) using generation with additional scoring-based quality evaluation and regeneration (+Regeneration). Similarly, Table \ref{tab:compare_TMO} compares the QAEA (TMO) performance between Single-generation and enhanced with Regeneration. Regeneration with scoring-based quality evaluation consistently improves retrieval performance across all cases and confirms the method's effectiveness.

\subsection{Ablation Study on QAEA}
\label{subsec:ablation}
We demonstrate the effectiveness of the QA and Event components of QAEA framework. From the equation \(\text{VDB}_{\text{final}} = \text{VDB}_{\text{ori}} + \text{VDB}_{\text{QA}} + \text{VDB}_{\text{EVENT}}\), we understand that QAEA incorporates three textual representations: original texts, QA pairs, and events. Hence, we consider seven scenarios derived from different combinations of these components: (1) Baseline (or Original, which includes only the original texts); (2) QA (only QA pair texts); (3) Event (only event texts); (4) Original+QA; (5) Original+Event; (6) QA+Event; (7) Original+QA+Event (complete QAEA implementation). 

\begin{figure}
  \centering
  \includegraphics[width=0.95\linewidth]{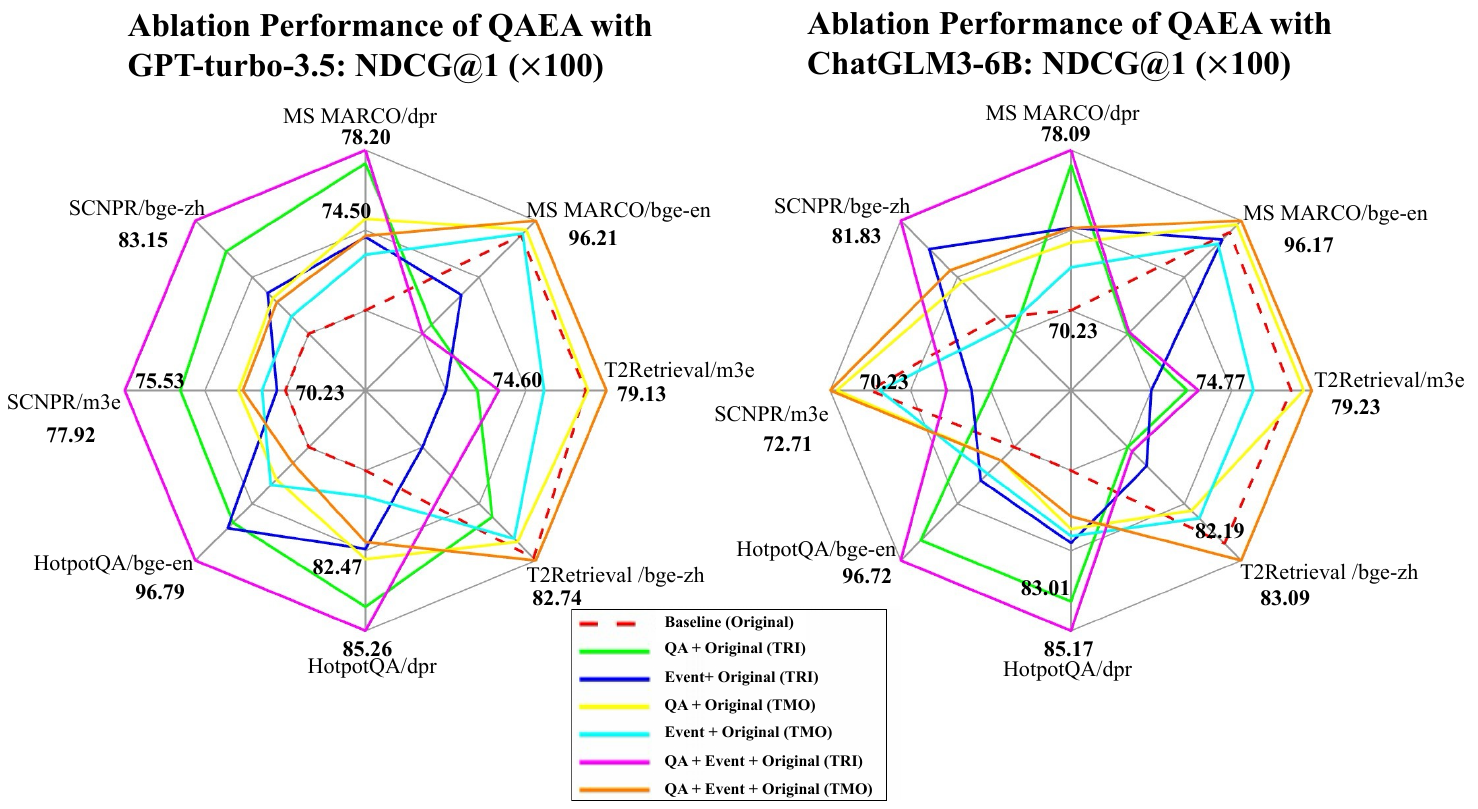}
  \caption{Ablation Performance of QAEA: NDCG@1 (\(\times 100\))}
  \label{experimental_analysis}
\end{figure}

Table \ref{tab:ablation} shows comparison results of the ablation study on QAEA. For discussion convenience, the table only show the QAEA (TMO) case and the conclusion also applies to QAEA (TRI). First of all, the complete QAEA implementation, Original+QA+Event, shows superior performance across most metrics compared to the subsets. When removing the QA component, the NDCG@1 average performance of QAEA drops by 0.15\%. When the Event component is removed, the NDCG@1 average performance drops by 0.90\%. Additionally, both Original+Event and Original+QA exceed the performance of the baselines, which further demonstrates the effectiveness and importance of both QA and Event components. We observe that Original+QA outperforms Original+Event; we believe this is because QA pairs are more likely to match the query both in form and semantics. Moreover, even without the original texts, single QA or Event components show performance close to the baselines. In the case of the MS MARCO dataset with the DPR model, they even outperform the baselines, which suggests the potential for these components to replace rather than merely augment the vector database.

Additionally, Fig. \ref{experimental_analysis} clearly shows the ablation performance of both QAEA (TMO) and QAEA (TRI). We observe that QA+Event+Original (TRI) or QA+Event+Original (TMO) achieve optimal performance, with many subsets also performing well.

\begin{table*}[t]
\centering
\caption{The NDCG (\(\times 100\)) performance comparison of different LLMs with QAEA (TMO)}
\label{tab:more_llm_expanded}
\resizebox{\textwidth}{!}{%
\begin{tabular}{l|cc|cc|cc|cc}
\hline
\multicolumn{1}{c|}{\multirow{3}{*}{Models}} 
 & \multicolumn{4}{c|}{HotpotQA} 
 & \multicolumn{4}{c}{MS MARCO} 
 \\
\cline{2-9}
 & \multicolumn{2}{c|}{dpr} 
 & \multicolumn{2}{c|}{bge-en} 
 & \multicolumn{2}{c|}{dpr} 
 & \multicolumn{2}{c}{bge-en} 
 \\
\cline{2-9}
 & NDCG@1 & NDCG@10 & NDCG@1 & NDCG@10 
 & NDCG@1 & NDCG@10 & NDCG@1 & NDCG@10 
 \\
\hline
Baseline & 79.38 & 74.24 & 95.79 & 92.70 & 70.23 & 80.64 & 95.93  & 98.02 \\ \hline
GPT-3.5-turbo & \textbf{81.97} & \textbf{76.01} & \textbf{96.06} & \textbf{92.91} & 74.33 & \textbf{83.80} & \textbf{96.21} & \textbf{98.19} \\
ChatGLM3-6B  & 81.66 & 75.96 & 95.99 & 92.83 & \textbf{74.96} & 83.72 & 96.17 & 98.15 \\
ChatGLM3-6B\(\dagger\)   & 80.66 & 75.41 & 95.87 & 92.79 & 74.15 & 83.68 & 96.04 & 98.12 \\
Mistral-7B\(\dagger\)    & 80.18 & 74.78 & 95.91 & 92.81 & 74.20 & 83.76 & 96.12  & 98.12 \\
Llama-3-8B\(\dagger\)    & 80.55 & 75.14 & 95.99 & 92.82 & 72.91 & 82.48 & 96.15  & 98.16 \\
\hline
\end{tabular}
}
\\[2pt]
\raggedright
\(\dagger\) denotes zero-shot prompting without optimization through evaluation and regeneration mechanisms. 
\end{table*}

\subsection{Case Study}
To better show how generated QA pairs and events address key information loss in long and noisy texts, we present the outputs of QAEA-DR via GPT-3.5-turbo on two cases from the MS MARCO dataset, as shown in Table \ref{tab:casestudy}. Table \ref{tab:casestudy} displays the generated texts with the highest score for each case.

In Case 1, the original text is relatively short but contains garbled texts ``Myers\textbackslash{}u00e2\textbackslash{}u0080\textbackslash{}u0093Briggs'' that does not match the query phrase ``Myers-briggs personality,'' which results in a low vector similarity score and ranking in retrieval. The generated QA pairs and events fix this noisy text issue via LLMs and extract structured key information to significantly boost similarity scores and retrieval rankings with the query. Besides, Case 2 presents an example of a longer text and the query relates to the closing house costs. The key phrase in the original text is ``Closing costs are roughly 2-2.25\% of the purchase price'' along with the keyword ``home.'' Due to the long original text, the term ``home'' is mentioned only once and may be lost during embedding. The generated QA pairs and events condense this key information and shorten the text, which enhances the fidelity of retrieval.

\subsection{More LLMs Generation Results}
\label{More LLMs Generation Results}
For further experiments on QAEA-DR with more open-source LLMs as text augmentation generators, we adopt Mistral-7B \cite{jiang2023mistral} and Llama-3-8B \cite{touvron2023llama}. Both Mistral and Llama use standard zero-shot prompts without evaluation and regeneration steps, while GPT-3.5-turbo and ChatGLM3-6B apply scoring-based evaluation and regeneration to improve output quality.

The results in Table \ref{tab:more_llm_expanded} show that GPT achieves the highest NDCG@1, especially with bge-en. GLM is slightly lower but still competitive. Mistral and Llama perform slightly weaker than GLM with regeneration. However, Mistral and Llama are likely to improve if they also support regeneration. Among open-source LLMs, ChatGLM3-6B is a strong choice for QAEA-DR. It supports both Chinese and English and has a smaller parameter size. All three open-source LLMs outperform the baselines and show performance close to GPT-3.5-turbo on English datasets. These findings confirm that small-scale open-source LLMs can provide effective text augmentation for QAEA-DR.

\begin{table}[]
\centering
\caption{Case study on QAEA-DR}
\label{tab:casestudy}
\resizebox{\columnwidth}{!}{%
\begin{tabular}{ll}
\hline
\textbf{Case1} & Myers-briggs personality test what do the letters stand for? \\ \hline
\multirow{2}{*}{\textbf{\begin{tabular}[c]{@{}l@{}}Gold Original\\ Text\end{tabular}}} & \begin{tabular}[c]{@{}l@{}}\{“text”: “INTP (introversion, intuition, thinking, perceiving)\\ is an abbreviation used in the publications of the   Myers\textbackslash{}u00e2\\ \textbackslash{}u0080\textbackslash{}u0093Briggs Type Indicator (MBTI) to refer to one of the\\ MBTI's 16 personality types. NTP (introversion, intuition, thinking,\\ perceiving) is an abbreviation used in the publications of the\\ Myers\textbackslash{}u00e2\textbackslash{}u0080\textbackslash{}u0093Briggs Type Indicator (MBTI) to refer to one\\ of the MBTI's 16 personality types.”\}\end{tabular} \\ \cline{2-2} 
 & \textbf{Score: 37.73    \ \ \ \ \ \ \ \ \        Rank: 324} \\ \hline
\multirow{2}{*}{\textbf{QA Pair}} & \begin{tabular}[c]{@{}l@{}}\{“text”: “What does INTP stand for in the context of the\\ Myers-Briggs Type Indicator (MBTI)? INTP stands for introversion, \\ intuition,   thinking, and perceiving, and it refers to one of the 16 \\ personality types in   the MBTI.”\}\end{tabular} \\
 & \textbf{Score: 72.59    \ \ \ \ \ \ \ \ \        Rank: 1} \\ \hline
\multirow{2}{*}{\textbf{Event}} & \begin{tabular}[c]{@{}l@{}}\{“text”: “Definition INTP abbreviation used in the publications\\ of the Myers Briggs Type Indicator (MBTI) to refer to one of the \\ MBTI's 16   personality types.”\}\end{tabular} \\ \cline{2-2} 
 & \textbf{Score: 61.08    \ \ \ \ \ \ \ \ \        Rank: 4} \\ \hline
 \hline
\textbf{Case2} & How much will i need to close on a house with fha? \\ \hline
\multirow{2}{*}{\textbf{\begin{tabular}[c]{@{}l@{}}Gold Original\\ Text\end{tabular}}} & \begin{tabular}[c]{@{}l@{}}\{“text”: “And then there are closing costs. Those two costs will\\ equal the total cash needed to close. Closing costs are roughly 2-2.25\% \\ of the purchase price. So if you’re looking at a home for \$500,000, plan on\\ closing costs to be around \$10,000.How much down payment … ” \}\\ (160 words)\end{tabular} \\ \cline{2-2} 
 & \textbf{Score: 36.32    \ \ \ \ \ \ \ \ \        Rank: 135} \\ \hline
\multirow{2}{*}{\textbf{QA Pair}} & \begin{tabular}[c]{@{}l@{}}\{“text”: “What are closing costs and how are they calculated?   Closing \\ costs are one of the two costs that equal the total cash needed to\\ close a property purchase. They are roughly 2-2.25\% of the purchase price.”\}\end{tabular} \\ \cline{2-2} 
 & \textbf{Score: 73.33    \ \ \ \ \ \ \ \ \        Rank: 2} \\ \hline
\multirow{2}{*}{\textbf{Event}} & \begin{tabular}[c]{@{}l@{}}\{“text”: “Closing Costs. Closing costs. purchase price. Closing\\ costs are roughly 2-2.25\% of the purchase price. For a home priced at\\ \$500,000, closing costs would be around \$10,000.”\}\end{tabular} \\ \cline{2-2} 
 & \textbf{Score: 76.04    \ \ \ \ \ \ \ \ \        Rank: 1} \\ \hline
\end{tabular}%
}
\end{table}

\section{Analysis}
In this section, we conduct the analysis of QAEA-DR and evaluate its performance under different conditions and configurations.

\begin{figure} [h]
  \centering
  \includegraphics[width=1.0\linewidth]{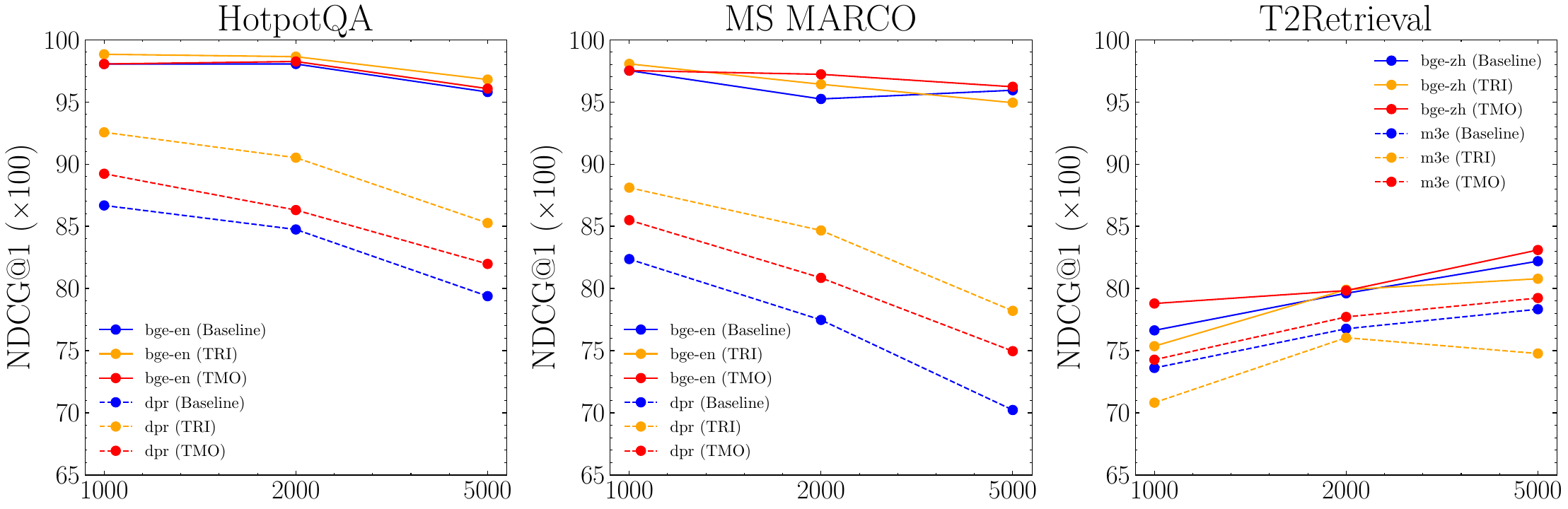}
  \caption{Analysis of different dataset sizes: NDCG@1 (\(\times 100\))}
  \label{perf_dataset_size}
\end{figure}

\begin{figure} [h]
  \centering
  \includegraphics[width=0.8\linewidth]{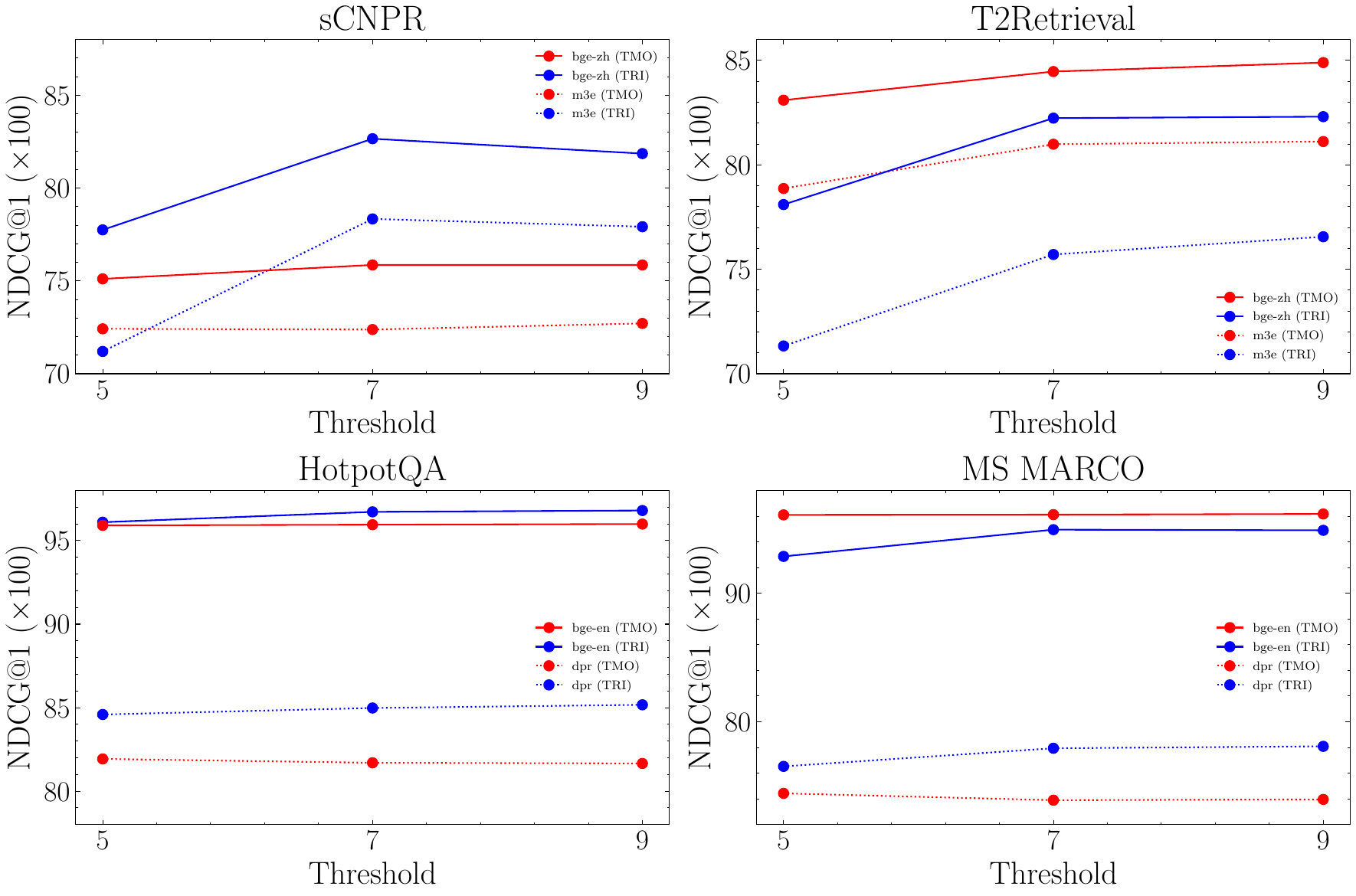}
  \caption{Analysis of different regeneration thresholds in QAEA-DR: NDCG@1 (\(\times 100\))}
  \label{analyse_regen_threashold}
\end{figure}

\subsection{Impact of Dataset Size}
We investigate the impact of dataset size across three open datasets as shown in Fig. \ref{perf_dataset_size}. Across dataset sizes of 1000, 2000, and 5000, QAEA-DR consistently demonstrates superior performance over the baseline. Besides, as dataset size varies, QAEA-DR and the baseline show similar trends across the three datasets, which indicates the consistency of text augmentation. The experimental results also corroborate our earlier discussions on the two text organization strategies. QAEA (TMO) shows good stability throughout the experiment, consistently outperforming the baseline. In contrast, although QAEA (TRI) shows significant performance improvements in most cases, its performance fluctuates greatly due to its high sensitivity to the quality of text generation, leading to unstable results across different dataset sizes.

\subsection{Impact of Regeneration Threshold}
We study the impact of the score threshold in text regeneration process. Higher threshold \(\tau\) means texts are more likely to require regeneration to meet these standards as shown in Algorithm \ref{alg1}.

Fig. \ref{analyse_regen_threashold} shows the retrieval performance at three different regeneration score thresholds: 5, 7, and 9, out of a total score of 10. Figure presents the best results from GLM and GPT. For the sCNPR and T2Retrieval datasets, NDCG@1 increases with higher thresholds, indicating that stricter filtering improves retrieval by regenerating lower-quality outputs. For HotpotQA and MS MARCO, performance remains relatively stable across threshold adjustments, which suggests that threshold changes have lesser impact due to the already high quality of the initial text generation. Meanwhile, QAEA (TMO) demonstrates more stable performance than QAEA (TMO) across datasets, particularly at higher thresholds.

\begin{table}
\centering
\caption{Average Number of QA Pairs and Events Generated From Original Text Through LLMs}
\label{tab:avg_num_qa_event}
\begin{tabular}{l|c|c|c}
\hline
Dataset            & Model      & Avg. Number of QA   & Avg. Number of Event \\
                   &            & per Text        & per Text \\
\hline
\multirow{2}{*}{sCNPR}       & GPT & 12.38 & 4.37 \\
                             & GLM & 11.06 & 4.18 \\
\hline
\multirow{2}{*}{T2Retrieval} & GPT & 6.13 & 3.71 \\
                             & GLM & 5.91 & 3.61 \\
\hline
\multirow{2}{*}{HotpotQA}    & GPT & 5.60 & 2.51 \\
                             & GLM & 5.07 & 2.35 \\
\hline
\multirow{2}{*}{MS MARCO}    & GPT & 4.31 & 2.27 \\
                             & GLM & 3.94 & 2.21 \\
\hline
\end{tabular}
\end{table}

\subsection{Impact of Generated Text Quantity}
We investigate the impact of the number of generated texts per original text on QAEA-DR performance. Table \ref{tab:avg_num_qa_event} summarizes the average number of QA pairs and events generated from original texts using GPT or GLM across four datasets. In all datasets, GPT generates more structured texts than GLM. Referring to Table \ref{tab:main_result}, where GPT generally outperforms GLM in QAEA (TRI) retrieval performance, we infer a positive correlation between the number of generated texts and retrieval performance. Table \ref{tab:avg_num_qa_event} also indicates that the average number of QA pairs generated is 2.15 times the average number of events. This finding is intuitive since each QA pair only conveys an individual information point, while event summarizes richer information.

\begin{table}[h]
\centering
\caption{Comparison of QAEA-DR generated text diversity scores}
\label{tab:text_diversity_comparison}
\resizebox{\columnwidth}{!}{%
\setlength{\tabcolsep}{3pt}
\small
\begin{tabular}{l|cccc|cccc}
\hline
\multirow{3}{*}{Metrics} & \multicolumn{4}{c|}{MS MARCO}                 & \multicolumn{4}{c}{HotpotQA}                \\
\cline{2-9}
                         & \multicolumn{2}{c|}{QA} & \multicolumn{2}{c|}{Event}  & \multicolumn{2}{c|}{QA} & \multicolumn{2}{c}{Event}  \\
\cline{2-9}
                         & GPT   & GLM   & GPT   & GLM   & GPT   & GLM   & GPT   & GLM   \\
\hline
Compression Ratio        & 1.0896 & 1.1016 & 1.2242 & 1.6411 & 0.9991 & 1.0112 & 1.1057 & 1.5182 \\
Self-BLEU                & 0.0934 & 0.0158 & 0.1457 & 0.4000 & 0.0414 & 0.0088 & 0.0867 & 0.3776 \\
Self-BERTScore               & 0.6848 & 0.6806 & 0.6662 & 0.8606 & 0.6355 & 0.6364 & 0.6006 & 0.8482 \\
Self-Repetition          & 0.1524 & 0.2024 & 0.0875 & 0.4786 & 0.0986 & 0.1623 & 0.0462 & 0.4595 \\
\hline
NDCG@1                   & 71.67  & 71.78  & 65.29  & 64.80  & 79.07  & 79.27  & 73.01  & 72.24  \\
\hline
\end{tabular}%
}
\\[2pt]
\raggedright
Lower diversity scores suggest greater text diversity.
\end{table}

\subsection{Impact of Generated Text Diversity}
\label{subsec:text diversity}
We perform a generated text diversity analysis to measure how the diversity of QA pairs and events affects retrieval performance, as shown in Table \ref{tab:text_diversity_comparison}. The analysis includes several metrics: Compression Ratio, Self-BLEU, Self-BERTScore, and Self-Repetition, where lower scores suggest greater text diversity \cite{shaib2024standardizing}.

The results show a clear correlation between higher diversity (lower scores) and improved retrieval performance. This effect is particularly evident in generated events across two datasets. Specifically, QA pairs exhibit greater diversity and better retrieval performance than events. For example, QA pairs have an average Self-BLEU score of 0.0399 with an NDCG@1 score of 75.35, while events have a higher Self-BLEU score of 0.2525 and an NDCG@1 score of 68.84. Additionally, although GPT generally shows more diversity than GLM across the two English datasets, their NDCG@1 scores remain comparable. This suggests that minor variations in text diversity have limited influence on retrieval performance, highlighting the robustness of QAEA-DR. We also observe that Self-BLEU strongly correlates with retrieval performance. Self-BLEU serves as a valuable metric for evaluating generated text quality, offering a lexical perspective alongside embeddings.

\begin{table}[h]
\centering
\caption{Comparison of LMQG and Our LLM-based GPT-QAG on Retrieval Performance: NDCG@1 (\(\times 100\))}
\label{tab:compare_with_QAEA_LMQG}
\begin{tabular}{l|ll|ll}
\hline
\multirow{2}{*}{} & \multicolumn{2}{c|}{MS MARCO} & \multicolumn{2}{c}{HotpotQA} \\ \cline{2-5} 
 & \multicolumn{1}{l|}{dpr} & bge-en & \multicolumn{1}{l|}{dpr} & bge-en \\ \hline
LMQG (TRI) & \multicolumn{1}{l|}{69.88} & 89.64 & \multicolumn{1}{l|}{78.66} & 93.55 \\
LMQG (TMO) & \multicolumn{1}{l|}{64.55} & 90.15 & \multicolumn{1}{l|}{77.49} & 93.90 \\
GPT-QAG (TRI) & \multicolumn{1}{l|}{\textbf{74.07}} & 92.89 & \multicolumn{1}{l|}{\textbf{80.23}} & 95.29 \\
GPT-QAG (TMO) & \multicolumn{1}{l|}{71.67} & \textbf{93.98} & \multicolumn{1}{l|}{79.07} & \textbf{95.33} \\ \hline
\end{tabular}
\\[2pt]
The values in bold represent the largest NDCG.
\end{table}

\subsection{QAG Comparisons}
Table \ref{tab:compare_with_QAEA_LMQG} presents the QAG retrieval performance of our LLM-based text augmentation method compared to the popular QAG model LMQG. We use the QA component of the QAEA framework as described in Section \ref{subsec:ablation} and employ GPT-3.5-turbo as the text generation LLM, denoted as GPT-QAG. Both LMQG and GPT-QAG generate vector databases of QA pair texts \(\text{VDB}_{\text{QA}}\) for dense retrieval. Since LMQG only supports English, we display comparison results on two English datasets. We observe that under both TRI and TMO QA text organization strategies, our GPT-QAG significantly outperforms LMQG, demonstrating the effectiveness of LLM-based QAG using a three-step prompting process. Notably, evaluating the text retrieval capability of generated QA pairs offers a new insight for assessing QAG performance. 

\begin{table}[h]
\centering
\caption{The NDCG (\(\times 100\)) and generated text noise (\%) performance on MS MARCO dataset under different input noise levels (\%)}
\label{tab:noise_sensitivity_analysis}
\begin{tabular}{l|c|c|c}
\hline
Method & Input Noise (\%) & NDCG@1 & Generated Text Noise (\%) \\ 
\hline
\hline
\multirow{4}{*}{QAEA-DR} 
 & 0   & 96.21 & --    \\ 
 & 20  & 96.12 & 0.33  \\ 
 & 50  & 96.15 & 0.80  \\ 
 & 100 & 96.15 & 1.94  \\ 
\hline
\multirow{4}{*}{QA} 
 & 0   & 93.98 & --    \\ 
 & 20  & 94.34 & 0.05  \\ 
 & 50  & 94.00 & 0.13  \\ 
 & 100 & 92.18 & 0.20  \\ 
\hline
\multirow{4}{*}{Event} 
 & 0   & 91.56 & --    \\ 
 & 20  & 92.12 & 0.79  \\ 
 & 50  & 91.78 & 1.86  \\ 
 & 100 & 91.86 & 4.38  \\ 
\hline
\end{tabular}
\end{table}

\subsection{Text Noise Sensitive Analysis}
We conduct a sensitivity analysis to evaluate the robustness of QAEA-DR against noise. We note that the noise in the texts refers to content that is unrelated to the query. In our experiments, we introduce semantically meaningless texts into the original texts to simulate random noise, defined as sequences of randomly generated letters, digits, and spaces (e.g., DtAfgJndLj FO7djxaOM). The original texts are augmented with 20\%, 50\%, and 100\% input noise based on their text length and processed through QAEA-DR to generate QA and event texts. We compare the NDCG@1 scores for retrieval with and without noise and calculate the proportion of input noise retained in the generated texts.

Table \ref{tab:noise_sensitivity_analysis} summarizes the results using GPT-3.5-turbo generator and the bge-en embedding model under the QAEA (TMO) strategy. The table shows that retrieval performance is not greatly affected by adding significant input noise. QA pairs and events retain some noise from the input, but the proportion of retained noise is much lower. For example, with 100\% input noise, the generated QA pairs retain 0.20\% of the noise, while the generated events retain 4.38\%. The higher noise retention in events is likely due to their structured complexity and longer descriptions. QA pairs focus on individual information points, which makes them less likely to retain noise.

\begin{table}[h]
\centering
\caption{Time Complexity Comparisons for QAEA (TMO) and QAEA (TRI) on sCNPR}
\label{tab:time_complexity_combined}
\begin{tabular}{l|c|c|c|c}
\hline
& \multicolumn{2}{c|}{TMO} & \multicolumn{2}{c}{TRI} \\ \hline
Baseline (Original) & \multicolumn{4}{c}{$nT_\text{sim}$} \\ \hline
& GPT & GLM & GPT & GLM \\
QA & $nT_\text{sim}$ & $nT_\text{sim}$ & $12.4nT_\text{sim}$ & $11.1nT_\text{sim}$ \\
Event & $nT_\text{sim}$ & $nT_\text{sim}$ & $4.4nT_\text{sim}$ & $4.2nT_\text{sim}$ \\
QA+Event & $2nT_\text{sim}$ & $2nT_\text{sim}$ & $16.8nT_\text{sim}$ & $15.3nT_\text{sim}$ \\
QA+Event+Original & $3nT_\text{sim}$ & $3nT_\text{sim}$ & $17.8nT_\text{sim}$ & $16.3nT_\text{sim}$ \\ \hline
\end{tabular}
\end{table}

\begin{table}[ht]
\centering
\caption{Computational cost comparison of different LLMs on MS MARCO dataset}
\label{tab:llm_comparison}
\begin{tabular}{lcccccc}
\hline
LLMs & Task & Size & NDCG@1 & Time & Memory \\
\hline
\hline
\multirow{2}{*}{GPT-3.5-turbo} & QAG    & \multirow{2}{*}{Unknown} & 93.98 & 3.29s  & \multirow{2}{*}{Unknown} \\
                               & EE     &                          & 91.56 & 2.69s  &                         \\
\cline{1-6}
\multirow{2}{*}{ChatGLM3}      & QAG    & \multirow{2}{*}{6B}      & 93.28 & 11.66s & \multirow{2}{*}{11.9GB} \\
                                & EE     &                          & 90.86 & 6.91s &                         \\
\cline{1-6}
\multirow{2}{*}{Mistral}       & QAG    & \multirow{2}{*}{7B}      & 93.23  & 15.26s & \multirow{2}{*}{13.8GB} \\
                                & EE     &                          & 92.01  & 6.58s &                         \\
\cline{1-6}
\multirow{2}{*}{Llama-3}       & QAG    & \multirow{2}{*}{8B}      & 92.89  & 10.39s & \multirow{2}{*}{15.3GB} \\
                                & EE     &                          & 87.67  & 6.05s &                         \\
\hline
\end{tabular}
\end{table}

\subsection{Retrieval Time Complexity}
The computational time of dense retrieval is independent of text length since all texts are embedded into vectors of the same dimensionality. Hence, the retrieval time is directly proportional to the number of text vectors in the vector database. For \(n\) texts and the associated vector database \(\text{VDB}_\text{ori} = \{v_1, v_2, \ldots, v_n\}\) under baseline, the time taken for each query to traverse the texts is \(n \times T_\text{sim}\), where \(T_\text{sim}\) is the computational cost to calculate \(\text{sim}(v_q, v_i)\), \(v_i \in \text{VDB}_\text{ori}\). Table \ref{tab:time_complexity_combined} presents the time complexity for QAEA (TMO) and QAEA (TRI) on the sCNPR dataset. Similar conclusions regarding time complexity are observed across other datasets.
\begin{itemize}
\item{For QAEA (TMO), each original text generates a merged QA pair text and a merged event text. In the TMO case, the presence of any component—Original, QA, or Event—increases the retrieval time by \(nT_\text{sim}\). It is observed that having only a QA or Event component without the original text results in computational times comparable to the baseline.}
\item{For QAEA (TRI), each original text generates and remains multiple QA pairs and events. In the TRI scenario, adding a QA or Event component may increase computational time several-fold compared to the baseline.}
\end{itemize}
To summarize, we consider the time overhead of QAEA-DR to be relatively flexible, and the individual QA or Event component in the TMO case matches the baseline time consumption, demonstrating a direction for optimizing text augmentation that balances high performance with time efficiency.

\subsection{LLM Computational Costs}
Table \ref{tab:llm_comparison} compares the computational costs of generating QAEA-DR texts for GPT-3.5-turbo, ChatGLM3-6B, Mistral-7B, and Llama-3-8B. The NDCG@1 results in Table \ref{tab:llm_comparison} are based on the bge-en embedding model under the QAEA (TMO) strategy. All experiments run on the NVIDIA RTX 3090 GPU (24 GB) without optimizations for inference efficiency or memory usage.

We observe that QAG tasks require more inference time than EE tasks, which suggests that QAG outputs are generally longer. This difference in output characteristics may contribute to QA pairs achieving higher NDCG@1 scores compared to events. GPT-3.5-turbo outperforms open-source small models in both accuracy and inference time, demonstrating its superior efficiency. For open-source LLMs, memory usage scales with model size, yet smaller-scale models still achieve reasonable inference time and solid retrieval performance. These findings show that QAEA-DR can run on moderate hardware in resource-limited scenarios, while larger models can improve generation quality. This flexibility allows users to balance computational cost and retrieval performance.

\section{Conclusion}
In this paper, we introduce QAEA-DR, a novel unified text augmentation framework for dense retrieval. This approach optimizes the original text by generating multiple QA pairs and events via LLM-based information extraction, which concentrates on key information and removes noisy text. As a result, the augmented vector database increases retrieval fidelity and effectively mitigates the issue of losing key information in dense retrieval. We conduct comprehensive experiments to demonstrate the effectiveness and robustness of QAEA-DR, even for datasets mainly comprising short texts. QAEA-DR indicates broader applicability by offering insights into open-domain LLM-based QAG and EE, and serving as a universal text optimizer in Retrieval-Augmented Generation (RAG).

\bibliographystyle{IEEEtran}
\bibliography{IEEEabrv, reference}

\begin{thebibliography}{10}
\providecommand{\url}[1]{#1}
\csname url@samestyle\endcsname
\providecommand{\newblock}{\relax}
\providecommand{\bibinfo}[2]{#2}
\providecommand{\BIBentrySTDinterwordspacing}{\spaceskip=0pt\relax}
\providecommand{\BIBentryALTinterwordstretchfactor}{4}
\providecommand{\BIBentryALTinterwordspacing}{\spaceskip=\fontdimen2\font plus
\BIBentryALTinterwordstretchfactor\fontdimen3\font minus \fontdimen4\font\relax}
\providecommand{\BIBforeignlanguage}[2]{{%
\expandafter\ifx\csname l@#1\endcsname\relax
\typeout{** WARNING: IEEEtran.bst: No hyphenation pattern has been}%
\typeout{** loaded for the language `#1'. Using the pattern for}%
\typeout{** the default language instead.}%
\else
\language=\csname l@#1\endcsname
\fi
#2}}
\providecommand{\BIBdecl}{\relax}
\BIBdecl

\bibitem{lee2019latent}
K.~Lee, M.-W. Chang, and K.~Toutanova, ``Latent retrieval for weakly supervised open domain question answering,'' in \emph{ACL - Annu. Meet. Assoc. Comput. Linguist., Proc. Conf.}, 2019, pp. 6086--6096.

\bibitem{karpukhin2020dense}
V.~Karpukhin, B.~Oguz, S.~Min, P.~Lewis, L.~Wu, S.~Edunov, D.~Chen, and W.-t. Yih, ``Dense passage retrieval for open-domain question answering,'' in \emph{EMNLP - Conf. Empir. Methods Nat. Lang. Process., Proc. Conf.}, 2020, pp. 6769--6781.

\bibitem{luan2021sparse}
Y.~Luan, J.~Eisenstein, K.~Toutanova, and M.~Collins, ``Sparse, dense, and attentional representations for text retrieval,'' \emph{Trans. Assoc. Comput. Linguist.}, vol.~9, pp. 329--345, 2021.

\bibitem{zhan2021jointly}
J.~Zhan, J.~Mao, Y.~Liu, J.~Guo, M.~Zhang, and S.~Ma, ``Jointly optimizing query encoder and product quantization to improve retrieval performance,'' in \emph{Int Conf Inf Knowledge Manage}, 2021, pp. 2487--2496.

\bibitem{li2023pseudo}
H.~Li, A.~Mourad, S.~Zhuang, B.~Koopman, and G.~Zuccon, ``Pseudo relevance feedback with deep language models and dense retrievers: Successes and pitfalls,'' \emph{ACM Trans. Inf. Syst.}, vol.~41, no.~3, pp. 1--40, 2023.

\bibitem{ni2022sentence}
J.~Ni, G.~H. Abrego, N.~Constant, J.~Ma, K.~Hall, D.~Cer, and Y.~Yang, ``Sentence-t5: Scalable sentence encoders from pre-trained text-to-text models,'' in \emph{Proc. Annu. Meet. Assoc. Comput Linguist.}, 2022, pp. 1864--1874.

\bibitem{yu2021improving}
H.~Yu, C.~Xiong, and J.~Callan, ``Improving query representations for dense retrieval with pseudo relevance feedback,'' in \emph{Int Conf Inf Knowledge Manage}, 2021, pp. 3592--3596.

\bibitem{johnson2019billion}
J.~Johnson, M.~Douze, and H.~J{\'e}gou, ``Billion-scale similarity search with gpus,'' \emph{{IEEE} Trans. Big Data}, vol.~7, no.~3, pp. 535--547, 2019.

\bibitem{xiao2023c}
S.~Xiao, Z.~Liu, P.~Zhang, and N.~Muennighof, ``C-pack: Packaged resources to advance general chinese embedding,'' \emph{arXiv:2309.07597}, 2023.

\bibitem{zhao2024retrieval}
P.~Zhao, H.~Zhang, Q.~Yu, Z.~Wang, Y.~Geng, F.~Fu, L.~Yang, W.~Zhang, and B.~Cui, ``Retrieval-augmented generation for ai-generated content: A survey,'' \emph{arXiv:2402.19473}, 2024.

\bibitem{wang2019query}
Y.~Wang, H.~Huang, and C.~Feng, ``Query expansion with local conceptual word embeddings in microblog retrieval,'' \emph{{IEEE} Trans. Knowl. Data Eng.}, vol.~33, no.~4, pp. 1737--1749, 2019.

\bibitem{gao2023precise}
L.~Gao, X.~Ma, J.~Lin, and J.~Callan, ``Precise zero-shot dense retrieval without relevance labels,'' in \emph{Proc. Annu. Meet. Assoc. Comput Linguist.}, 2023, pp. 1762--1777.

\bibitem{bonifacio2022inpars}
L.~Bonifacio, H.~Abonizio, M.~Fadaee, and R.~Nogueira, ``Inpars: Unsupervised dataset generation for information retrieval,'' in \emph{SIGIR - Proc. Int. ACM SIGIR Conf. Res. Dev. Inf. Retr.}, 2022, pp. 2387--2392.

\bibitem{yang2014detecting}
Y.~Yang and A.~Nenkova, ``Detecting information-dense texts in multiple news domains,'' in \emph{Proc Natl Conf Artif Intell}, vol.~28, no.~1, 2014.

\bibitem{green1961baseball}
B.~F. Green~Jr, A.~K. Wolf, C.~Chomsky, and K.~Laughery, ``Baseball: an automatic question-answerer,'' in \emph{Proc. Western Jt. Computer Conference: Extending Man's Intellect, IRE-AIEE-ACM}, 1961, pp. 219--224.

\bibitem{cohen2023qa}
W.~W. Cohen, W.~Chen, M.~De~Jong, N.~Gupta, A.~Presta, P.~Verga, and J.~Wieting, ``Qa is the new kr: question-answer pairs as knowledge bases,'' in \emph{Proc. AAAI Conf. Artif. Intell., AAAI}, vol.~37, no.~13, 2023, pp. 15\,385--15\,392.

\bibitem{lee2023read}
K.~Lee, S.~E. Han, S.~W. Hwang, and M.~Lee, ``When to read documents or qa history: On unified and selective open-domain qa,'' in \emph{Proc. Annu. Meet. Assoc. Comput Linguist.}\hskip 1em plus 0.5em minus 0.4em\relax Association for Computational Linguistics (ACL), 2023, pp. 6420--6432.

\bibitem{xiang2019survey}
W.~Xiang and B.~Wang, ``A survey of event extraction from text,'' \emph{{IEEE} Access}, vol.~7, pp. 173\,111--173\,137, 2019.

\bibitem{vaswani2017attention}
A.~Vaswani, N.~Shazeer, N.~Parmar, J.~Uszkoreit, L.~Jones, A.~N. Gomez, {\L}.~Kaiser, and I.~Polosukhin, ``Attention is all you need,'' \emph{Adv. neural inf. proces. syst.}, vol.~30, 2017.

\bibitem{devlin2019bert}
J.~Devlin, M.-W. Chang, K.~Lee, and K.~Toutanova, ``Bert: Pre-training of deep bidirectional transformers for language understanding,'' in \emph{NAACL HLT - Conf. N. Am. Chapter Assoc. Comput. Linguistics: Hum. Lang. Technol. - Proc. Conf.}, 2019, pp. 4171--4186.

\bibitem{raffel2020exploring}
C.~Raffel, N.~Shazeer, A.~Roberts, K.~Lee, S.~Narang, M.~Matena, Y.~Zhou, W.~Li, and P.~J. Liu, ``Exploring the limits of transfer learning with a unified text-to-text transformer,'' \emph{J. Mach. Learn. Res.}, vol.~21, no.~1, pp. 5485--5551, 2020.

\bibitem{yates2021pretrained}
A.~Yates, R.~Nogueira, and J.~Lin, ``Pretrained transformers for text ranking: Bert and beyond,'' in \emph{SIGIR - Proc. Int. ACM SIGIR Conf. Res. Dev. Inf. Retr.}, 2021, pp. 1154--1156.

\bibitem{lewis2020retrieval}
P.~Lewis, E.~Perez, A.~Piktus, F.~Petroni, V.~Karpukhin, N.~Goyal, H.~K{\"u}ttler, M.~Lewis, W.-t. Yih, T.~Rockt{\"a}schel \emph{et~al.}, ``Retrieval-augmented generation for knowledge-intensive nlp tasks,'' \emph{Adv. neural inf. proces. syst.}, vol.~33, pp. 9459--9474, 2020.

\bibitem{zheng2020fast}
C.~Zheng, L.~Zhu, X.~Lu, J.~Li, Z.~Cheng, and H.~Zhang, ``Fast discrete collaborative multi-modal hashing for large-scale multimedia retrieval,'' \emph{{IEEE} Trans. Knowl. Data Eng.}, vol.~32, no.~11, pp. 2171--2184, 2020.

\bibitem{reimers2019sentence}
N.~Reimers and I.~Gurevych, ``Sentence-bert: Sentence embeddings using siamese bert-networks,'' in \emph{EMNLP-IJCNLP - Conf. Empir. Methods Nat. Lang. Process. Int. Jt. Conf. Nat. Lang. Process., Proc. Conf.}, 2019, pp. 3982--3992.

\bibitem{muennighoff2023mteb}
N.~Muennighoff, N.~Tazi, L.~Magne, and N.~Reimers, ``Mteb: Massive text embedding benchmark,'' in \emph{EACL - Conf. Eur. Chapter Assoc. Comput. Linguist., Proc. Conf.}, 2023, pp. 2014--2037.

\bibitem{izacard2021unsupervised}
G.~Izacard, M.~Caron, L.~Hosseini, S.~Riedel, P.~Bojanowski, A.~Joulin, and E.~Grave, ``Unsupervised dense information retrieval with contrastive learning,'' \emph{Proc. Mach. Learn. Res.}, 2021.

\bibitem{huang2023make}
R.~Huang, J.~Huang, D.~Yang, Y.~Ren, L.~Liu, M.~Li, Z.~Ye, J.~Liu, X.~Yin, and Z.~Zhao, ``Make-an-audio: Text-to-audio generation with prompt-enhanced diffusion models,'' in \emph{Proc. Mach. Learn. Res.}\hskip 1em plus 0.5em minus 0.4em\relax PMLR, 2023, pp. 13\,916--13\,932.

\bibitem{lu2022reacc}
S.~Lu, N.~Duan, H.~Han, D.~Guo, S.-w. Hwang, and A.~Svyatkovskiy, ``Reacc: A retrieval-augmented code completion framework,'' in \emph{Proc. Annu. Meet. Assoc. Comput Linguist.}, 2022, pp. 6227--6240.

\bibitem{cowie1996information}
J.~Cowie and W.~Lehnert, ``Information extraction,'' \emph{Commun ACM}, vol.~39, no.~1, pp. 80--91, 1996.

\bibitem{ji2021survey}
S.~Ji, S.~Pan, E.~Cambria, P.~Marttinen, and S.~Y. Philip, ``A survey on knowledge graphs: Representation, acquisition, and applications,'' \emph{{IEEE} Trans. Neural Netw. Learn. Syst.}, vol.~33, no.~2, pp. 494--514, 2021.

\bibitem{wadden2019entity}
D.~Wadden, U.~Wennberg, Y.~Luan, and H.~Hajishirzi, ``Entity, relation, and event extraction with contextualized span representations,'' in \emph{EMNLP-IJCNLP - Conf. Empir. Methods Nat. Lang. Process. Int. Jt. Conf. Nat. Lang. Process., Proc. Conf.}, 2019, pp. 5784--5789.

\bibitem{valenzuela2015domain}
M.~A. Valenzuela-Esc{\'a}rcega, G.~Hahn-Powell, M.~Surdeanu, and T.~Hicks, ``A domain-independent rule-based framework for event extraction,'' in \emph{ACL-IJCNLP - Annu. Meet. Assoc. Comput. Linguist. Int. Jt. Conf. Nat. Lang. Process., Proc. Syst. Demonstr.}, 2015, pp. 127--132.

\bibitem{chen2015event}
Y.~Chen, L.~Xu, K.~Liu, D.~Zeng, and J.~Zhao, ``Event extraction via dynamic multi-pooling convolutional neural networks,'' in \emph{ACL-IJCNLP - Annu. Meet. Assoc. Comput. Linguist. Int. Jt. Conf. Nat. Lang. Process. Asian Fed. Nat. Lang. Process., Proc. Conf.}, 2015, pp. 167--176.

\bibitem{nguyen2016joint}
T.~H. Nguyen, K.~Cho, and R.~Grishman, ``Joint event extraction via recurrent neural networks,'' in \emph{Conf. North Am. Chapter Assoc. Comput. Linguist.: Hum. Lang. Technol., NAACL HLT - Proc. Conf.}, 2016, pp. 300--309.

\bibitem{wei2023zero}
X.~Wei, X.~Cui, N.~Cheng, X.~Wang, X.~Zhang, S.~Huang, P.~Xie, J.~Xu, Y.~Chen, M.~Zhang \emph{et~al.}, ``Zero-shot information extraction via chatting with chatgpt,'' \emph{arXiv:2302.10205}, 2023.

\bibitem{huang2023event}
H.~Huang, X.~Liu, G.~Shi, and Q.~Liu, ``Event extraction with dynamic prefix tuning and relevance retrieval,'' \emph{{IEEE} Trans. Knowl. Data Eng.}, vol.~35, no.~10, pp. 9946--9958, 2023.

\bibitem{liu2018jointly}
X.~Liu, Z.~Luo, and H.-Y. Huang, ``Jointly multiple events extraction via attention-based graph information aggregation,'' in \emph{Proc. Conf. Empir. Methods Nat. Lang. Process., EMNLP}, 2018, pp. 1247--1256.

\bibitem{wang2023document}
X.~Wang, L.~Gui, and Y.~He, ``Document-level multi-event extraction with event proxy nodes and hausdorff distance minimization,'' in \emph{Proc. Annu. Meet. Assoc. Comput Linguist.}, 2023, pp. 10\,118--10\,133.

\bibitem{lewis2019unsupervised}
P.~Lewis, L.~Denoyer, and S.~Riedel, ``Unsupervised question answering by cloze translation,'' in \emph{ACL - Annu. Meet. Assoc. Comput. Linguist., Proc. Conf.}, 2019, pp. 4896--4910.

\bibitem{zhang2019addressing}
S.~Zhang and M.~Bansal, ``Addressing semantic drift in question generation for semi-supervised question answering,'' in \emph{EMNLP-IJCNLP - Conf. Empir. Methods Nat. Lang. Process. Int. Jt. Conf. Nat. Lang. Process., Proc. Conf.}, 2019, pp. 2495--2509.

\bibitem{heilman2010good}
M.~Heilman and N.~A. Smith, ``Good question! statistical ranking for question generation,'' in \emph{NAACL HLT - Hum. Lang. Technol.: Annu. Conf. North Am. Chapter Assoc. Comput. Linguist., Proc. Conf.}, 2010, pp. 609--617.

\bibitem{lee2020generating}
D.~B. Lee, S.~Lee, W.~T. Jeong, D.~Kim, and S.~J. Hwang, ``Generating diverse and consistent qa pairs from contexts with information-maximizing hierarchical conditional vaes,'' \emph{Proc. Annu. Meet. Assoc. Comput Linguist.}, pp. 208 -- 224, 2020.

\bibitem{ushio2023empirical}
A.~Ushio, F.~Alva-Manchego, and J.~Camacho-Collados, ``An empirical comparison of lm-based question and answer generation methods,'' in \emph{Proc. Annu. Meet. Assoc. Comput Linguist.}, 2023, pp. 14\,262--14\,272.

\bibitem{giray2023prompt}
L.~Giray, ``Prompt engineering with chatgpt: a guide for academic writers,'' \emph{Ann Biomed Eng}, vol.~51, no.~12, pp. 2629--2633, 2023.

\bibitem{xie2023t2ranking}
X.~Xie, Q.~Dong, B.~Wang, F.~Lv, T.~Yao, W.~Gan, Z.~Wu, X.~Li, H.~Li, Y.~Liu \emph{et~al.}, ``T2ranking: A large-scale chinese benchmark for passage ranking,'' in \emph{SIGIR - Proc. Int. ACM SIGIR Conf. Res. Dev. Inf. Retr.}, 2023, pp. 2681--2690.

\bibitem{nguyen2016ms}
T.~Nguyen, M.~Rosenberg, X.~Song, J.~Gao, S.~Tiwary, R.~Majumder, and L.~Deng, ``Ms marco: A human generated machine reading comprehension dataset,'' \emph{CEUR Workshop Proc.}, vol. 2640, p. 660, 2016.

\bibitem{yang2018hotpotqa}
Z.~Yang, P.~Qi, S.~Zhang, Y.~Bengio, W.~Cohen, R.~Salakhutdinov, and C.~D. Manning, ``Hotpotqa: A dataset for diverse, explainable multi-hop question answering,'' in \emph{Proc. Conf. Empir. Methods Nat. Lang. Process., EMNLP}, 2018, pp. 2369--2380.

\bibitem{du2022glm}
Z.~Du, Y.~Qian, X.~Liu, M.~Ding, J.~Qiu, Z.~Yang, and J.~Tang, ``Glm: General language model pretraining with autoregressive blank infilling,'' in \emph{Proc. Annu. Meet. Assoc. Comput Linguist.}, 2022, pp. 320--335.

\bibitem{jiang2023mistral}
A.~Q. Jiang, A.~Sablayrolles, A.~Mensch, C.~Bamford, D.~S. Chaplot, D.~d.~l. Casas, F.~Bressand, G.~Lengyel, G.~Lample, L.~Saulnier \emph{et~al.}, ``Mistral 7b,'' \emph{arXiv:2310.06825}, 2023.

\bibitem{touvron2023llama}
H.~Touvron, L.~Martin, K.~Stone, P.~Albert, A.~Almahairi, Y.~Babaei, N.~Bashlykov, S.~Batra, P.~Bhargava, S.~Bhosale \emph{et~al.}, ``Llama 2: Open foundation and fine-tuned chat models,'' \emph{arXiv:2307.09288}, 2023.

\bibitem{shaib2024standardizing}
C.~Shaib, J.~Barrow, J.~Sun, A.~F. Siu, B.~C. Wallace, and A.~Nenkova, ``Standardizing the measurement of text diversity: A tool and a comparative analysis of scores,'' \emph{arXiv:2403.00553}, 2024.

\end{thebibliography}

\vspace{-0.5cm}

\begin{IEEEbiography}[{\includegraphics[width=0.8in,height=1in,clip,keepaspectratio]{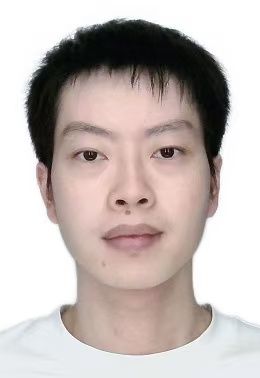}}]{Hongming Tan}
received the Master of Arts degree in Statistics from Washington University in St. Louis, U.S. in 2020. He is currently working towards the PhD degree with Tsinghua University. His research interests include natural language processing, deep learning, and information retrieval.
\end{IEEEbiography}

\vspace{-5cm}

\begin{IEEEbiography}[{\includegraphics[width=0.8in,height=1in,clip,keepaspectratio]{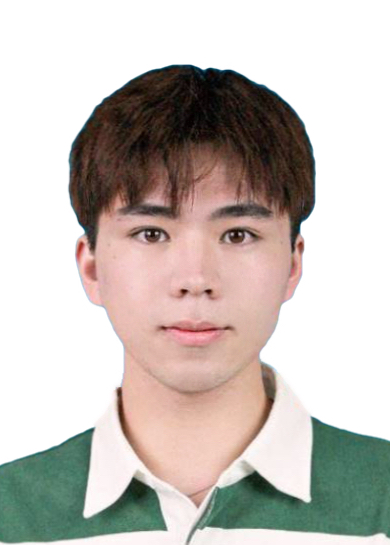}}]{Shaoxiong Zhan}
received the Bachelor of Engineering degree from Huazhong Agricultural University, China. He is currently working towards the master’s degree with Tsinghua University. His research interests include information retrieval, large language models, and natural language processing.
\end{IEEEbiography}

\vspace{-5cm}

\begin{IEEEbiography}[{\includegraphics[width=0.8in,height=1in,clip,keepaspectratio]{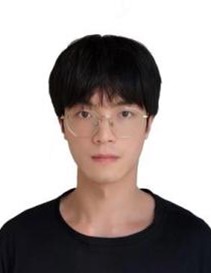}}]{Hai Lin}
received the Master's degree at the School of Computer, Electronics, and Information, Guangxi University in Nanning, China. Currently, he is pursuing a Ph.D. at Tsinghua University, Shenzhen International Graduate School in Shenzhen, Guangdong, China. His research interests include natural language processing, large language models, dialogue systems, and multimodal models.
\end{IEEEbiography}

\vspace{-5cm}

\begin{IEEEbiography}[{\includegraphics[width=0.8in,height=1in,clip,keepaspectratio]{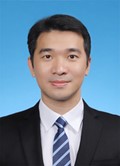}}]{Hai-Tao Zheng}
(Senior Member, IEEE) received the Ph.D. degree in medical informatics from Seoul National University, South Korea, in 2009. He is currently an Associate Professor with the Tsinghua Shenzhen International Graduate School, Tsinghua University, China. His research interests include knowledge engineering, natural language processing, and large language models.
\end{IEEEbiography}

\vspace{-5cm}

\begin{IEEEbiography}[{\includegraphics[width=1in,height=1.2in,clip,keepaspectratio]{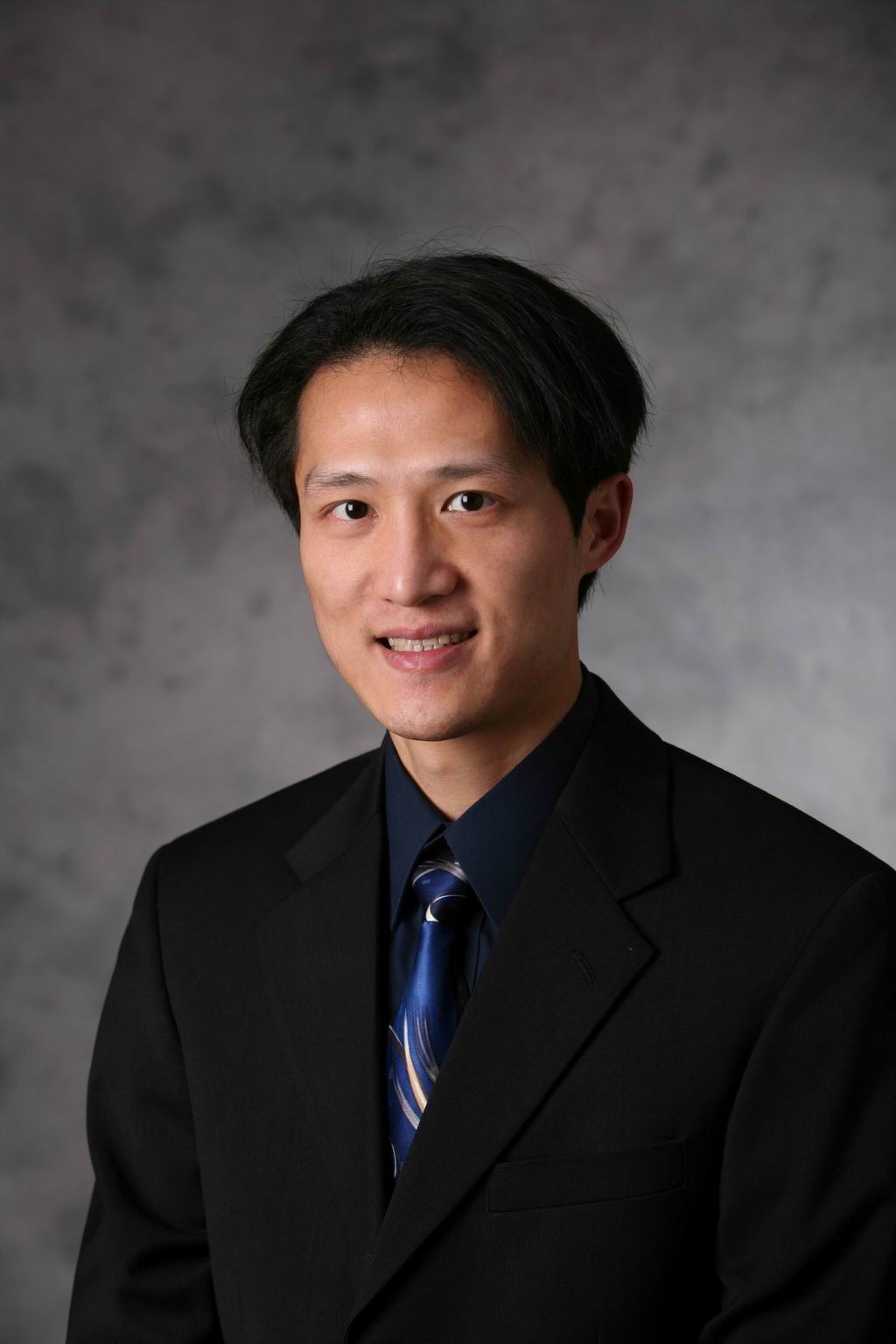}}]{Wai Kin (Victor) Chan}
(Senior Member, IEEE) received the Ph.D. degree in industrial engineering and operations research from the University of California, Berkeley. He is Professor of the Shenzhen International Graduation School (SIGS), Tsinghua University, China. His research interests include agent-based and discrete-event simulation, data science and operations research, and their applications in emergency management, Edverse, social networks, service systems, healthcare, transportation, energy markets, and manufacturing. 
\end{IEEEbiography}


\end{document}